%% file: main.tex
\input{macro}
\documentclass{article}

 \usepackage[main,final]{neurips_2026}

\usepackage{microtype}
\usepackage{graphicx}
\usepackage{subcaption}
\usepackage{booktabs} 
\usepackage[utf8]{inputenc} 
\usepackage[T1]{fontenc}    
\usepackage{hyperref}       
\usepackage{url}            
\usepackage{booktabs}       
\usepackage{amsfonts}       
\usepackage{nicefrac}       
\usepackage{multirow}
\usepackage{microtype}      
\usepackage{xcolor}         
\usepackage{amsmath}
\usepackage{amssymb}
\usepackage{mathtools}
\usepackage{amsthm}
\usepackage{algorithm} 
\usepackage{algpseudocode}
\usepackage{float}
\usepackage{booktabs,graphicx,caption}
\usepackage{makecell} 
\usepackage{subcaption}
\usepackage{placeins}
\usepackage[capitalize,noabbrev]{cleveref}
\usepackage{dblfloatfix}
\usepackage{wrapfig}
\usepackage{etoc}
\theoremstyle{plain}
\newtheorem{theorem}{Theorem}[section]

\newtheorem{lemma}[theorem]{Lemma}

\theoremstyle{definition}

\newtheorem{assumption}[theorem]{Assumption}
\theoremstyle{remark}
\newtheorem{remark}[theorem]{Remark}

\usepackage[textsize=tiny]{todonotes}
\title{Generative OOD-regularized Model-based \\ Policy Optimization}

\author{
\begin{tabular}{cccc}
Aysin Tumay$^{1}$ &
Jiahe Huang$^{1}$ &
Elise Jortberg$^{2}$ &
Rose Yu$^{1}$
\end{tabular}
\vspace{0.5em}\\
$^{1}$University of California, San Diego \qquad
$^{2}$Abiomed
}

\begin{document}
\etocdepthtag.toc{mtchapter}

\maketitle

\begin{abstract}


We study sequential decision-making with offline reinforcement learning (RL). Traditional offline RL policies may result in out-of-distribution (OOD) actions when training relies only on sparse offline representations. To ensure safe offline policies in a sparse state-action space, we explore how density estimation models can be integrated into model-based RL methods to avoid OOD regions. Generative models are capable of explicitly modeling the density in sparse state-action spaces. Building on this, we introduce \textbf{G}enerative \textbf{O}OD-\textbf{r}egularized \textbf{M}odel-based \textbf{P}olicy \textbf{O}ptimization (GORMPO), a density-regularized offline RL algorithm that uses generative density modeling to restrict policy updates to high-density areas of the dataset. Furthermore, we examine whether \textit{better OOD detection corresponds to better model-based offline policies}. We compare (1) the OOD detection capabilities of various density estimators and (2) their performance within the GORMPO framework on a real-world medical dataset and sparse offline RL datasets.
We provide theoretical guarantees bounding GORMPO's performance gap in terms of density estimation error, ensuring near-optimal policies under imperfect density estimates.
Empirically, GORMPO remains competitive in reward and improves clinical metrics compared to state-of-the-art clinical RL models, and outperforms state-of-the-art offline RL models in offline RL benchmarks. Our empirical findings show that better OOD detection generally results in improved policies in environments with stable dynamics, while conservative penalties with poor density estimation are favored when dynamics are uncertain. Code, checkpoints, sparse and OOD D4RL datasets are available at \url{https://github.com/Rose-STL-Lab/GORMPO}.

\end{abstract}

\vspace{-0.5em}
\section{Introduction}

\input{sections/intro}
\vspace{-0.5em}
\section{Related Work}
\input{sections/related_work}


\vspace{-0.5em}
\section{Methodology}

\input{sections/methodology}
\section{Theoretical Results}
\input{sections/theory}
\section{Results}
\input{sections/experiments}

\section{Conclusion}
\input{sections/conclusion}

\newpage
\begin{ack}
This work was supported in part by a research grant
from Abiomed, Inc., a Johnson \& Johnson company,
the U.S. Army Research Office under Army-ECASE
award W911NF-07-R-0003-03, the U.S. Department
Of Energy, Office of Science, IARPA HAYSTAC
Program, and NSF Grants \#2205093, \#2146343,
\#2134274, CDC-RFA-FT-23-0069, DARPA AIE
FoundSci and DARPA YFA.

\textbf{Institutional Review Board (IRB) Disclosure}
 This study used retrospective, de-identified patient records under IRB approval.
 
 \textbf{Competing interests.} E.J. is an employee of Abiomed, Inc., which manufactures
the Impella device studied in this work. The remaining
authors declare no competing interests.
\end{ack}

\bibliography{example_paper}
\bibliographystyle{plain}

\newpage
\appendix
\onecolumn
\etocdepthtag.toc{mtappendix}
\etocsettagdepth{mtchapter}{none}
\etocsettagdepth{mtappendix}{subsection}
\renewcommand{\contentsname}{Appendix Contents}
\tableofcontents
\clearpage
\input{sections/appendix}

\clearpage

\end{document}

%% file: macro.tex
\newcommand{\result}[2]{%
  \ensuremath{#1\,{\scriptstyle \pm\,#2}}%
}
\newcommand{\bestresult}[2]{%
  \ensuremath{\mathbf{#1}\,{\scriptstyle \mathbf{\pm\,#2}}}%
}

%% file: sections/intro.tex

Offline reinforcement learning (RL) has shown great promise in safety-critical control tasks where interaction with an online environment is impossible. Applications of this problem have been demonstrated in sepsis \cite{komorowski2018artificial, raghu2017continuous} and cancer prediction \cite{eckardt2021association, tseng2017machine}, autonomous driving \cite{bansal2019chauffeurnet}, and Unmanned Aerial Vehicle control \cite{brunke2022safe}. 
However, distribution shift constitutes a significant problem for offline RL because the policy reaches out-of-distribution (OOD) regions, leading to unsafe actions.
Offline RL also depends strongly on dataset coverage where limited state–action support leads to extrapolation error in Bellman backups \cite{kumar2019bear}. This issue is common in clinical settings because most patients receive adequate support and remain hemodynamically stable, while cases of insufficient treatment are rare and underrepresented as seen in Figure \ref{fig:sparsity}.


\begin{figure}[t]
    \centering
    
    \begin{subfigure}[b]{0.37\linewidth}
        \centering
        \includegraphics[width=\linewidth]{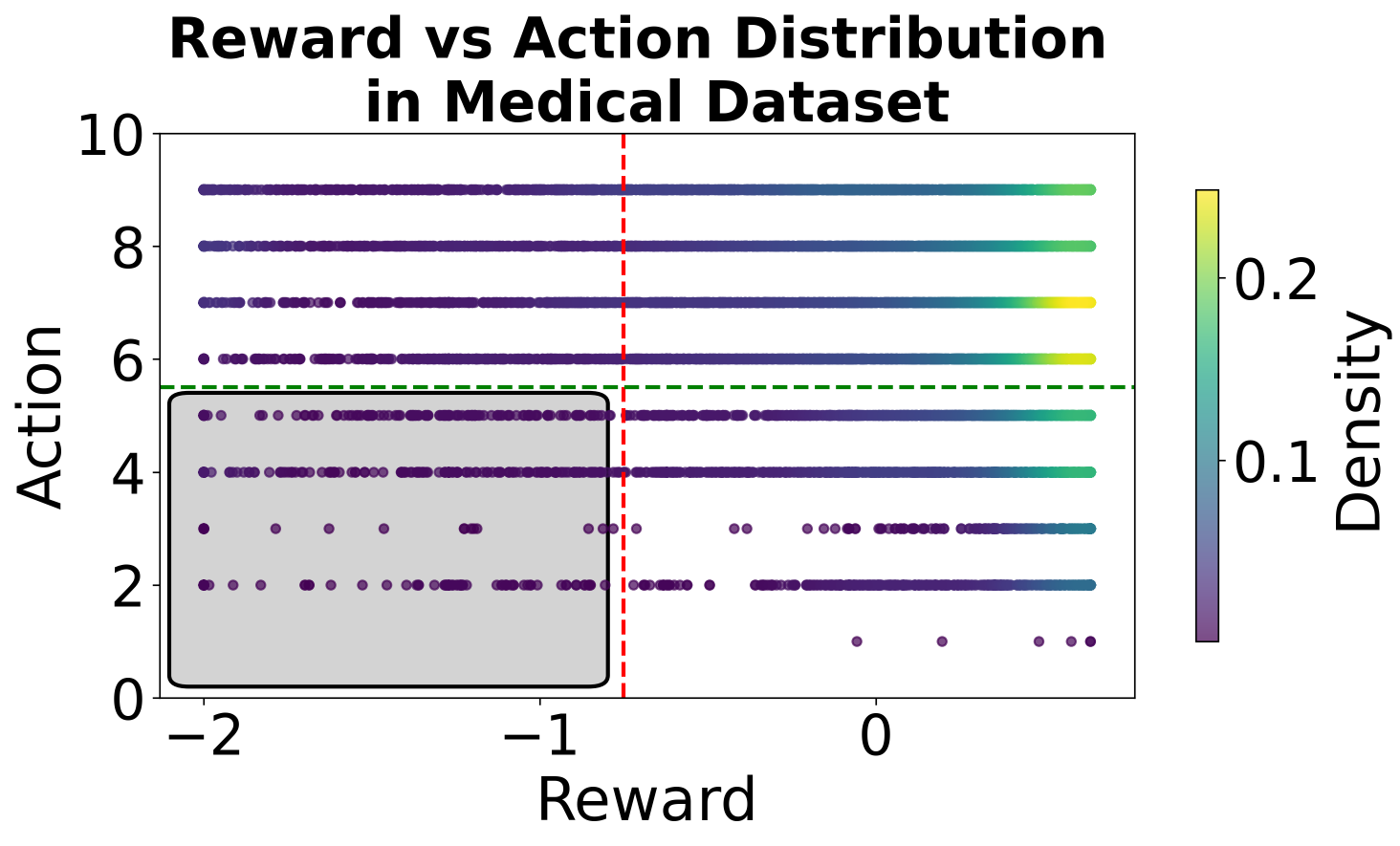}
        \label{fig:sparsity_left}
    \end{subfigure}
    \hspace{-0.4em}
    \begin{subfigure}[b]{0.61\linewidth}
        \centering
        \includegraphics[width=\linewidth]{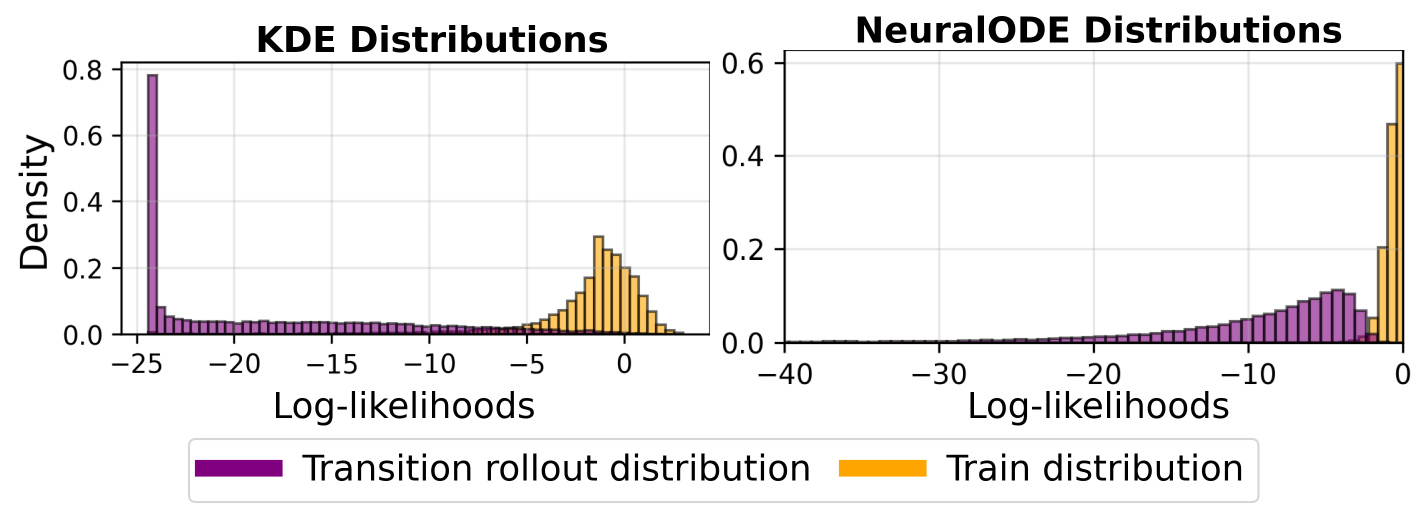}
        \label{fig:sparsity_right}
    \end{subfigure}

    \vspace{-1em}
    \caption{
    Left: reward–action space of the medical dataset, with the sparse region in gray. Right: transition rollout and training distributions modeled by KDE and NeuralODE. Rollouts with low likelihoods indicate OOD behavior, while KDE saturates with a large spike.
    }
    
    \label{fig:sparsity}
    \vspace{-1em}
\end{figure}

A core challenge in offline RL is mitigating distribution shift without being overly conservative. Prior model-free methods constrain the learned policy to the behavior distribution on the offline data only \cite{fujimoto2019bcq, kumar2020conservative}. 
However, these constraints potentially suppress valid actions where the Q-function generalizes well~\cite{an2021uncertainty}.
Other model-free methods \cite{wu2022spot, svr} address this by explicitly regularizing the policy with behavior support based on the evidence lower bound (ELBO), or with importance sampling \cite{svr}, which might fall short on highly sparse offline data.
Some model-based uncertainty-penalized methods \cite{yu2020mopo, pmlr-v202-sun23q}, built on model-based policy optimization (MBPO) \cite{janner2019mbpo}, address distribution shift by penalizing policy optimization with uncertainty. 
   However, augmenting the dataset with model-generated rollouts can  drift into OOD states when the dynamics model is trained on sparse data.
      Prior work uses KDE-based regularization in healthcare problems: OGSRL \citep{ogsrl} imposes an OOD cost in a constrained MDP, whereas CORMPO \citep{tumay2025guardian} combines rollout regularization with clinical reward shaping. 
      However, generative density estimation offers a more principled alternative for flexible OOD detection compared to KDE in sparse and high-stakes tasks such as healthcare \cite{zhai, melychuk, Ruff_2021}.

In this work, we propose \textbf{G}enerative \textbf{O}OD-\textbf{r}egularized \textbf{M}odel-based \textbf{P}olicy \textbf{o}ptimization (\textbf{GORMPO}), a density-penalized model-based offline RL algorithm. To avoid the uncovered areas in the next state-action space during dynamics model rollouts, we discount the rolled-out rewards based on the next state-action density. Lower-density areas get penalized more to support in-distribution (ID) data augmentation during policy optimization. On top of this, since our pipeline is highly flexible, we calibrated five different density estimation models to answer if  \textit{better OOD detection leads to better offline policies} when used in GORMPO.

In summary, our primary contributions are:
\vspace{-0.5em}
\begin{itemize}
    \item GORMPO, a plug-and-play framework that uses generative models to penalize OOD next-state–action pairs in model-based offline RL.
It outperforms state-of-the-art baselines in reward on sparse benchmarks and clinical metrics on real-world medical data.
    \item Theoretical guarantees that account for reward model errors, unlike prior model-based offline RL work.
    \item The first comprehensive evaluation of 5 distinct families of density estimators with four generative models for the OOD detection task.
    \item Demonstrating expressive density models lead to better policies in stable dynamics, whereas pessimistic regularization is required for uncertain dynamics.
\end{itemize}
\vspace{-0.5em}


%% file: sections/related_work.tex
\textbf{Constrained Offline RL.}
Addressing distribution shift is central to offline reinforcement learning. A prominent class of methods mitigates extrapolation error by constraining the learned policy to the support of the behavior policy. Approaches such as Batch-Constrained Q-learning (BCQ) \cite{fujimoto2019bcq}, Conservative Q-Learning (CQL) \cite{kumar2020conservative}, and Behavior Regularized Actor-Critic (BRAC) \cite{wu2019behavior} employ divergence metrics
to penalize deviations from the offline dataset. While effective at ensuring safety, distance-based constraints often induce excessive conservatism. By restricting learning to the neighborhood of observed data, they may suppress potentially optimal actions that remain plausible within the true underlying distribution \cite{an2021uncertainty, degrave2022magnetic}. 
To overcome these limitations, recent work has shifted to explicit regularization in offline RL.  
Model-free methods regularize the policy toward the behavior distribution \cite{wu2022spot}, reinforcing the patterns in the offline dataset during optimization. However, estimating the lower bound provides a coarse regularization.
With explicit density estimation in model-free RL, CPED \cite{zhang2023cped} integrates FlowGAN, highlighting the benefits of generative density models. Nevertheless,  relying on the data support with weak expert demonstrations can still yield suboptimal policies. In this case, MBPO \cite{janner2019mbpo} proves that learning a dynamics model helps in broader exploration of the state-action space. Some impose pessimism through uncertainty penalties \cite{pmlr-v202-sun23q, yu2020mopo}, quantifying uncertainty across an ensemble of dynamics models and penalizing value estimation and dynamics rollouts. 
SAMBO-RL \cite{sambo} introduces a shift-aware reward correction to mitigate model bias. However, uncertainty penalties, bias corrections, and value regularization can still be insufficient if rollouts drift into OOD regions when the dynamics model is trained on sparse data, requiring a density constraint.

Density regularization has effectively suppressed OOD behavior in offline RL:
OGSRL \cite{ogsrl} embeds an OOD cost constraint into an estimated constrained Markov Decision Process, using a support classifier to restrict state-action trajectories to data support. 
In contrast, rather than deriving KDE-based sample-complexity bounds
like OGSRL, we bound performance in terms of a generic density error, which allows GORMPO to use more expressive generative density models that better capture high-dimensional, sparse state–action domains than KDE.
On KDE-based penalties, CORMPO \cite{tumay2025guardian} applies two-sided regularization to model rollouts alongside clinical reward shaping. GORMPO instead adopts a penalty-only formulation, derives tighter bounds than CORMPO, and analyzes how different density estimators influence policy optimization.

\textbf{Generative Density Estimators.}
Generative density estimation has advanced through three primary methodological paradigms. First, normalizing flows \cite{rezende2015variational} construct complex distributions from simple base measures via sequences of invertible transformations, facilitating exact likelihood evaluation. Second, diffusion and score-based generative models \cite{song2020denoising} model distributions through a forward noise-injection and learned reverse process. Recent iterations, such as improved DDPMs \cite{ho2020denoising} and EDM \cite{Karras2022edm}, have further optimized log-likelihood estimation and sample fidelity. To address the computational cost of high-dimensional data, Latent Diffusion Models \cite{rombach2022high} and autoregressive methods \cite{six2025nepa} shift the generative process into a compressed latent space. Third, high-capacity backbones such as ViT \cite{dosovitskiy2020image} and DiT \cite{peebles2023scalable} have demonstrated remarkable success in modeling complex manifolds. When combined with continuous-time frameworks like Neural ODEs \cite{chen2018neural}, they offer a powerful inductive bias for capturing the underlying dynamics of temporal trajectories.
Most recently, unified frameworks such as Flow Matching \cite{lipman2022flow} and one-step approaches like MeanFlow \cite{geng2025mean} have sought to reconcile flow and diffusion methods, significantly accelerating inference while maintaining expressive density estimation.
However, while individual classes of generative models have been applied in isolation, such as diffusion for planning \cite{janner2022diffuser, wang2023diffusion} or normalizing flows for policy constraints \cite{akimov2022let}, a systematic evaluation of modern density estimators specifically for OOD detection within model-based offline RL remains underexplored.





%% file: sections/methodology.tex
\label{methodology}


\begin{wrapfigure}{r}{0.60\linewidth}
\centering
\vspace{-1em}
    \includegraphics[
        width=\linewidth,
        clip,
        trim=5450 0 5570 0
    ]{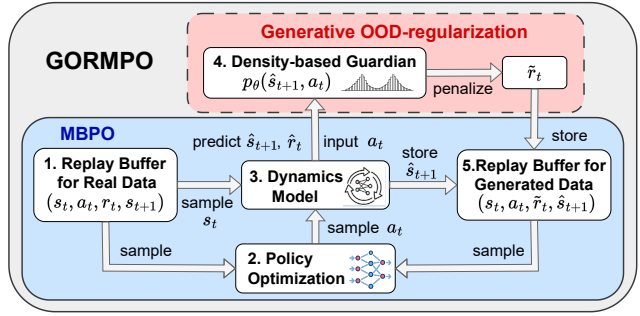}

    \vspace{0.5em}

    \caption{
    \textbf{System diagram of GORMPO.} 
    We enhance MBPO (blue module) through generative OOD regularization (red module). 
    We sample an action $a_t$ and use the pretrained dynamics model to predict the next state $\hat s_{t+1}$ and reward $\hat r_t$. 
    We then compute the likelihood of $(\hat s_{t+1}, a_t)$ under a pretrained generative density estimator and penalize $\hat r_t$ in low-density regimes, producing $\tilde r_t$. 
    Finally, we store $(s_t, a_t, \tilde{r}_t, \hat s_{t+1})$ in the generated data buffer.
    }

    \label{fig:system_diagram}
    \vspace{-2em}
\end{wrapfigure}

We propose GORMPO, a generative density estimator-based OOD regularization for model-based RL (see Figure \ref{fig:system_diagram} and the algorithm pseudocode in Appendix \ref{algo}). 
In this section, we formulate the problem and detail components of GORMPO.

\vspace{-0.5em}
\subsection{Problem Setting}

\paragraph{Markov Decision Process.} We formulate our setting as a Markov decision process (MDP), defined by the tuple $\mathcal{M} = (\mathcal{S}, \mathcal{A}, T, r, \mu_0, \gamma)$,
with state space $\mathcal{S}$, action space $\mathcal{A}$, transition dynamics $T: \mathcal{S}\times \mathcal{A} \rightarrow \Delta (\mathcal{S})$, reward function $r(s,a): \mathcal{S}\times \mathcal{A} \rightarrow \mathbb{R}$, initial state distribution $\mu_0$, and discount factor, $\gamma \in (0,1)$.  RL algorithms aim to find a policy $\pi : \mathcal{S} \rightarrow  \mathcal{A}$ that maximizes the expected cumulative discounted reward
$\mathbb{E}_{\pi, s_0\sim \mu_0} \left[ \sum_{t=0}^{\infty} \gamma^t r(s_t, a_t) \right]$ where $s_0$ is the initial state. The optimal policy is defined as,
\begin{equation}
    \pi^* =\arg\max_{\pi} \mathbb{E}_{\pi, s_0\sim \mu_0} \left[ \sum_{t=0}^{\infty} \gamma^t r(s_t, a_t) \right].
\label{eq:offline_rl}
\end{equation}
The \textit{Offline RL setting} is when the algorithm only has access to a dataset sampled from the environment $\mathcal{D}_{\text{env}} = \{(s, a, r, s')\}_{t=1}^{N}$ collected by a behavior policy and cannot interact with the environment.
%
%

\paragraph{Density Estimator Formulation. }Our objective for density estimation is to probabilistically model the density of \texttt{next state + action} pairs.
Our generative model learns an approximation $p_\theta(x_{t})$ of the density where $x_t = (s_{t+1}, a_t)$ at time $t$ by maximizing the expected log‑likelihood on the training samples as 
$$\max_\theta \; \mathbb{E}_{x_t \sim \mathcal{D}_{\text{env}}} \big[ \log p_\theta(x_{t}) \big].$$
By doing so, the model becomes capable of estimating how likely a candidate pair of next state and action is under the learned dynamics distribution.

\subsection{Density-based guardian for Offline RL}

Our goal is to learn a policy $\pi$ that maximizes the expected return in the true MDP $\mathcal{M} = (\mathcal{S}, \mathcal{A}, T, r, \mu_0, \gamma)$.
We denote the value function under the true dynamics as $V^{\pi}_{\mathcal{M}}(s)$ and the expected return as \[\eta_{\mathcal{M}}(\pi) = \mathbb{E}_{s_0 \sim \mu_0}[V^{\pi}_{\mathcal{M}}(s_0)].\] 
Following the model-based approach, we jointly learn a transition dynamics model $\hat T_\phi(s'\mid s,a)$ and a reward distribution $\hat r_\psi(s,a)$ from $\mathcal{D}_{\text{env}}$.
Model rollouts are generated by sampling $\hat s_{t+1}\sim \hat T_\phi(\cdot\mid s_t,a_t)$ and $\hat{R}_t\sim\hat{r}_\psi(\cdot\mid s_t,a_t)$.
Throughout, $r(s,a)$ denotes the true conditional expected
immediate reward, and
$\hat{r}_\psi(s,a)
:= \mathbb{E}_{\hat{R}\sim
\hat{r}_\psi(\cdot\mid s,a)}[\hat{R}]$
denotes its learned counterpart. Subsequently, we define the model MDP as $\hat{\mathcal{M}} = (\mathcal{S}, \mathcal{A}, \hat{T}_\phi, \hat{r}_\psi, \mu_0, \gamma)$. We drop the parameters $\phi$ and $\psi$ for the ease of notation.
Let 
\begin{equation}
    \rho^{\pi}_{\hat{T}}(s_t,a_t) = \pi(a_t|s_t)\sum_{t'=0}^{\infty} \gamma^{t'} P^{\pi}_{\hat{T}}(s_{t'})
\end{equation}
denote the unnormalized discounted occupancy measure under policy $\pi$ and dynamics $\hat{T}$, where $P^{\pi}_{\hat{T}}(s_t)$ is the probability of visiting state $s$ at time $t$. We use $\rho^\pi$ for $\rho^{\pi}_{\hat{T}}(s_t,a_t)$ in Section \ref{sec:theory} for ease of notation.


Our model-based offline RL architecture leverages data augmentation with the rollouts of a learned dynamics model, which exhibits varying degrees of accuracy across the state-action space.
This directly controls the support of the training data during policy optimization. While the optimal policy under perfect dynamics may venture beyond the behavioral distribution to achieve higher returns, model inaccuracies in these out-of-distribution (OOD) regions can lead to catastrophic failures in high-stakes medical tasks. 

To account for OOD rollouts during data augmentation, the density-based guardian (red module in Figure \ref{fig:system_diagram}) distinguishes estimated in-distribution (ID) and OOD dynamics rollouts. Estimating the probability of $(\hat{s}_{t+1}, a_t)$ through different families of generative models is non-trivial since 
(1) substantially varying likelihood scales across model families make the penalty level difficult to control (see Appendix \ref{sec:more_rl} for likelihood scales) 
(2) not all generative models are inherently density estimators, e.g., score-matching-based diffusion models.
To address these challenges, we penalize the rolled-out reward by employing a tanh penalty, which ensures training stability and strict theoretical guarantees. 


Given the density estimator, $p_\theta$, we define the estimated density regularizer as:
\begin{equation}
\hat{u}(\hat s_{t+1},a_t) = \tanh(\max(\tau - \log (p_\theta(\hat{s}_{t+1},a_t)), 0))
\label{eq:penalty}
\end{equation}
at time step $t$, where $\tau$ is the density threshold of ID data. 
$\tau$ is a hyperparameter chosen based on the validation set. 
Now define the regularized MDP as $\tilde{\mathcal{M}} = (\mathcal{S}, \mathcal{A}, \hat{T}, \tilde{r}, \mu_0, \gamma)$ with 
\begin{equation}
\tilde{r}(s_{t},a_t) = \hat{r}(s_{t},a_t) - \lambda \hat{u}(\hat s_{t+1},a_t).
\label{eq:reg}
\end{equation}
where $\lambda = \gamma c \cdot C_{\hat{T}}$, introduced in the next section. $C_{\hat{T}}$ establishes a link between the dynamics model error and the density-based penalty. The optimal policy for the density-penalized MDP is obtained by solving
\begin{equation}
\hat{\pi} = \arg\max_{\pi} \eta_{\tilde{\mathcal{M}}}(\pi).
\label{eq:final_opt}
\end{equation}

We discuss some other density formulations in Appendix \ref{sec:more_rl}.
Next, we explain the formulation of the density estimators incorporated in the GORMPO pipeline.






\subsection{Generative Density Estimators}
\label{sec:generative_models}
In this section, we detail the specific density estimation models $ p_\theta(\mathbf{x})$ employed within the GORMPO framework. For all models, we treat the next state-action pair $\mathbf{x} = (\mathbf{s'}, \mathbf{a})$ as the joint input variable where $\mathbf{s} \in \mathbb{R}^{N \times d_1}$, $\mathbf{a} \in \mathbb{R}^{N \times d_2}$, and $\mathbf{x} \in \mathbb{R}^{N \times ( d_1+ d_2)}$, where $N$ is the sample size.
\paragraph{Kernel Density Estimation (KDE).}
We utilize KDE as a non-parametric baseline to benchmark against generative approaches. The density estimate is given by
$p_{\text{KDE}}(\mathbf{x}) = \frac{1}{N} \sum_{i=1}^N K_h(\mathbf{x} - \mathbf{x}_i)$,
where $K_h$ is the Gaussian kernel with bandwidth $h$, and $\{\mathbf{x}_i\}_{i=1}^N$ represents the training dataset $\mathcal{D}_{\text{env}}$. To ensure computational efficiency in training, we approximate the summation with $k$-nearest neighbors of $\mathbf{x}$.


\paragraph{Variational Autoencoders (VAE).}
We adapt the standard VAE \cite{vae} to the continuous state-action space. The model maximizes the ELBO on the log-likelihood:
\begin{equation}
    \log p_\theta(\mathbf{x}) \geq \mathbb{E}_{q_\phi(\mathbf{z}|\mathbf{x})} [\log p_\theta(\mathbf{x}|\mathbf{z})] - D_{\text{KL}}(q_\phi(\mathbf{z}|\mathbf{x}) \| p(\mathbf{z})),
\end{equation}
where the prior is $p(\mathbf{z}) = \mathcal{N}(\mathbf{0}, \mathbf{I})$. The approximate posterior is parameterized as a diagonal Gaussian $q_\phi(\mathbf{z}|\mathbf{x}) = \mathcal{N}(\boldsymbol{\mu}_\phi(\mathbf{x}), \text{diag}(\boldsymbol{\sigma}^2_\phi(\mathbf{x})))$. At inference, we utilize the importance-weighted estimate \cite{burda}.


\paragraph{Normalizing Flows (RealNVP).}
To enable exact log-likelihood computation, we employ the RealNVP architecture~\cite{dinh2017realnvp}. The model learns a bijection $f_\theta: \mathcal{X} \to \mathcal{Z}$ mapping the complex state-action distribution to a standard Gaussian latent space $\mathcal{Z}$. The log-density is computed via the change-of-variables formula:
\begin{equation}
    \log p_\theta(\mathbf{x}) = \log p_{\mathcal{Z}}(f_\theta(\mathbf{x})) + \log \left| \det \frac{\partial f_\theta(\mathbf{x})}{\partial \mathbf{x}^T} \right|.
\end{equation}
We utilize coupling layers to ensure the Jacobian determinant is computationally tractable.



\paragraph{Diffusion Probabilistic Models.}\label{diffusion-method}
We model the density of $\mathbf{x}$ using a Denoising Diffusion Probabilistic Model (DDPM)~\citep{ho2020denoising}. The forward process $q(\mathbf{x}^{(k)} | \mathbf{x}^{(k-1)})$ incrementally adds Gaussian noise over $K$ steps. We learn a reverse process parameterized by a noise prediction network $\boldsymbol{\epsilon}_\theta(\mathbf{x}^{(k)}, k)$ trained to minimize:
\begin{equation}
    \mathcal{L}_{\text{simple}} = \mathbb{E}_{\mathbf{x}^{(0)}, \boldsymbol{\epsilon}, k} \left[ \| \boldsymbol{\epsilon} - \boldsymbol{\epsilon}_\theta(\sqrt{\bar{\alpha}_k}\mathbf{x}^{(0)} + \sqrt{1-\bar{\alpha}_k}\boldsymbol{\epsilon}, k) \|^2 \right].
\end{equation}

Exact likelihood computation requires evaluating the ELBO over the full trajectory ($K \approx 1000$ steps), which is computationally expensive for online OOD detection. To ensure tractability, we estimate the likelihood using a \textit{strided inference schedule}, summing the variational terms over a subsampled sequence of timesteps $\mathcal{S} \subset \{1, \dots, K\}$ where $|\mathcal{S}| \ll K$ (e.g., $20$--$50$ steps). This reduces the inference complexity from $O(K)$ to $O(|\mathcal{S}|)$ while preserving sufficient accuracy to distinguish ID transitions from OOD. Each term in the strided ELBO is a KL divergence between Gaussian transitions, estimated by Monte Carlo sampling of the noise $\mathbf{\epsilon}$.

\paragraph{Continuous Normalizing Flows (Neural ODE).}
As a continuous-time counterpart to normalizing flows, we model the evolution of the density using a Neural ODE~\citep{chen2018neural}. The state-action pair is treated as the terminal state of a continuous transformation $\mathbf{z}(\nu)$, $\nu \in [0,1]$ governed by the ODE $d\mathbf{z}(\nu)/d\nu = f_\theta(\mathbf{z}(\nu), \nu)$. The log-density is obtained by integrating the instantaneous change of variables:
\begin{equation}
    \log p_\theta(\mathbf{x}) = \log p_0(\mathbf{z}(0)) - \int_0^1 \text{Tr}\left( \frac{\partial f_\theta}{\partial \mathbf{z}(\nu)} \right) d\nu.
\end{equation}
This formulation allows for smooth density estimation that respects the continuous nature of physical control tasks.



%% file: sections/theory.tex
\label{sec:theory}
In this section, we theoretically show that GORMPO has a guaranteed performance in the true $\mathcal{M}$ and achieves near-optimal performance under the learned dynamics. 
Our guarantee relies on four assumptions. Assumption \ref{ass1} on bounded rewards is standard in RL theory and holds for most practical applications; Assumption \ref{ass2} captures epistemic uncertainty in supervised learning, where model performance degrades as we move away from the training distribution. Assumption \ref{ass3} captures the quality of the density estimation model. Finally, Assumption \ref{ass4:reward} uniformly bounds the error of the learned reward model across all state–action pairs.
For ease of notation, we denote the next state and action at any $t$ as $s'$ and $a$, respectively.

\begin{assumption}[Bounded Rewards and Value Functions]
\label{ass1}
The reward function is bounded: $|r(s,a)| \leq r_{\max}$ for all $(s,a) \in \mathcal{S} \times \mathcal{A}$. Consequently, $V^\pi_{\mathcal{M}} \in c\mathcal{F}$ where $\mathcal{F} = \{f : \|f\|_\infty \leq 1\}$ and $c = r_{\max}/(1-\gamma)$.
\end{assumption}

\begin{assumption}[Density-Dependent Model Error]
\label{ass2}
There exists a constant $C_{\hat{T}} > 0$ such that for all $(s,a) \in \mathcal{S} \times \mathcal{A}$:
\begin{equation}
d_{\mathcal{F}}(\hat{T}(s,a), T(s,a)) \leq C_{\hat{T}} \cdot \big[\mathbb{E}_{s' \sim \hat{T}(s,a)} [u(s',a)]\big]  + \epsilon_{\text{approx}} 
\end{equation}
where $d_{\mathcal{F}}$ is the integral probability metric w.r.t. $\mathcal{F}$, and $\epsilon_{\text{approx}} > 0$ represents an irreducible approximation error.
\end{assumption}

$C_{\hat T}$ depends on the model class capacity and the smoothness of the true dynamics, discussed in Appendix \ref{proofs}.

\begin{assumption}[Density Estimation Error]
\label{ass3}
There exists a constant $\epsilon_{\text{density}} > 0$ such that for all $(s',a) \in \mathcal{S} \times \mathcal{A}$:
\begin{equation}
\left|\log p_\theta(s',a) - \log p(s',a)\right| \leq \epsilon_{\text{density}}
\end{equation}
where $p_\theta(s',a)$ is the learned density estimator and $p(s',a)$ is the true behavioral density from the offline dataset.
\end{assumption}

\begin{assumption}[Reward Model Error]
\label{ass4:reward}
There exists a constant $\epsilon_{\mathrm{reward}} > 0$ such that for all
$(s,a) \in \mathcal{S} \times \mathcal{A}$:
\begin{equation}
|r(s,a) - \hat{r}_\psi(s,a)| \leq \epsilon_{\mathrm{reward}}
\end{equation}
where $\hat{r}_\psi(s,a)$ is the learned reward model and $r(s,a)$ is the true
reward function.
\end{assumption}




With Theorem \ref{thm:1}, we demonstrate that the performance lower bound on the true MDP degrades linearly with both the density estimation error $\epsilon_{\text{density}}$ and the reward model error $\epsilon_{\text{reward}}$. Reducing either error therefore tightens the guarantee, holding other terms fixed. In Theorem \ref{thm:thm2}, we show that the policy learned under imperfect density estimation is competitive with the best policy satisfying a relaxed density constraint, with an optimality gap that also grows linearly with both errors. The relaxation is proportional to $\epsilon_{\text{density}}$ and $\epsilon_{\text{reward}}$, meaning better density and reward estimation allows comparison with a less conservative constraint set. 

\begin{theorem}[Conservative Value Bound]
\label{thm:1}
Under Assumptions~\ref{ass1}-\ref{ass4:reward}, for any stationary policy $\pi$,
\[
\eta_{\mathcal M}(\pi)
\ge
\eta_{\tilde{\mathcal M}}(\pi)
-
\frac{
\gamma c\,\epsilon_{\mathrm{approx}}
+\lambda\,\epsilon_{\mathrm{density}}
+\epsilon_{\mathrm{reward}}
}{1-\gamma}.
\]
where $\lambda = \gamma c\, C_{\hat{T}}$ is the penalty weight.
\end{theorem}

\begin{theorem}[Optimality Gap]
\label{thm:thm2}
Under Assumptions~\ref{ass1}-\ref{ass4:reward}, let
$\lambda=\gamma c C_{\hat T}$ and define
\[
J_u(\pi)
:=
\mathbb{E}_{(s,a)\sim\rho^\pi_{\hat T}}
\!\Big[
\mathbb{E}_{s'\sim\hat T(s,a)}[u(s',a)]
\Big],
\]
where $u$ is the true-density regularizer.

For any $\delta\geq0$ such that the comparison class below
is nonempty,
\[
\eta_{\mathcal M}(\hat\pi)
\geq
\sup_{\pi:\,J_u(\pi)\leq
\delta+\frac{\epsilon_{\mathrm{density}}}{1-\gamma}}
\eta_{\mathcal M}(\pi)
-2\lambda\delta
-\frac{
4\lambda\epsilon_{\mathrm{density}}
+2\gamma c\,\epsilon_{\mathrm{approx}}
+2\epsilon_{\mathrm{reward}}
}{1-\gamma}.
\]
\end{theorem}


%
We improve the theoretical results of  CORMPO \citep{tumay2025guardian} with stronger guarantees in $\mathcal{M}$, and better near-optimal policies maintaining a low penalty in $\mathcal{\tilde{M}}$. 
Our analysis additionally accounts for reward-model error. The telescoping argument underlying model-based offline RL guarantees, including MOPO \cite{yu2020mopo} and CORMPO \cite{tumay2025guardian}, is stated for MDPs sharing a reward function, while the corresponding algorithms predict reward jointly with the next state from the learned dynamics model. We therefore generalize to MDPs differing in both reward and dynamics and introduce Assumption \ref{ass4:reward} bounding $\epsilon_{\mathrm{reward}}$ which contributes an additive 
$\frac{\epsilon_{\mathrm{reward}}}{1-\gamma}$  term to our bounds.
See the proofs, lemmas, and further discussion on assumptions in Appendix \ref{proofs}.

%% file: sections/experiments.tex

In this section, we present our experiment setup; results on the real-world medical dataset and sparse D4RL datasets. See Appendix \ref{sec:ablations} for hyperparameter sensitivity and penalty type ablations.
\vspace{-0.5em}
\subsection{Experiment Setup}


\paragraph{Dataset and Task.} We use a medical dataset that includes cardiogenic shock patient physiologic features and mechanical circulatory support inputs. Cardiogenic Shock (CGS) is a syndrome characterized by cardiac output insufficient for end-organ perfusion, where mechanical circulatory support (MCS) devices play an integral role in aiding heart muscle recovery \cite{vahdatpour2019cardiogenic,tumay2025guardian}.  
Our task is to learn a safe MCS weaning policy by reducing the pump level (P-level) while maintaining stable hemodynamics.  This task reflects sparse state-action coverage because few patients receive insufficient MCS and remain unstable, which leaves the low P-level and low reward region in Figure \ref{fig:sparsity} underrepresented.
 

\paragraph{Baselines.}
Our base model is MBPO \cite{janner2019mbpo}, which uses a learned dynamics model to generate short synthetic rollouts that are then mixed with real environment data to train a standard model-free policy. We also use MOBILE \cite{pmlr-v202-sun23q}, a state-of-the-art (SOTA) model-based policy that penalizes the Bellman equation based on the disagreement in the dynamics models, as a base policy for sparse D4RL experiments. 
For SOTA model-free policies, we select SPOT \cite{wu2022spot} that uses a VAE to penalize the actor network with ELBO loss. We additionally provide performance results for CORMPO \cite{tumay2025guardian} with clinical reward shaping and two-sided KDE penalties, and GMB-TRPO from OGSRL \cite{ogsrl} with a density-based constrained MDP on our medical task.

\paragraph{Hyperparameter Tuning.} Based on \cite{DBLP:conf/nips/BrandfonbrenerW21, pmlr-v162-kurenkov22a}, we finetune $\lambda$ from a small set of hyperparameters by allowing access to the online simulation environment for 100 episodes (see Appendix \ref{sec:model_params}). 

\paragraph{OOD Experiment Setup.} We investigate the OOD detection performance of KDE, VAE, RealNVP, diffusion, and NeuralODE models. The goal is to resemble OOD transition dynamics during policy training since a controlled OOD detection experiment directly on transition model rollouts is not trivial. 
We generate OOD datasets by concatenating the datasets with their Gaussian noise $\mathcal{N}(\mu, 0.1)$-added versions (see Appendix \ref{sec:ood_data}).
The mean of the added noise is labeled as ``OOD Distance'' in Figures \ref{fig:tnr_tpr} and \ref{fig:ood_experiments}. Smaller OOD distance values correspond to a harder task due to the poor separation.
For evaluation, we employ accuracy, true positive rate (TPR), and true negative rate (TNR) based on the threshold selection on the validation set and ROC AUC for OOD-ID separation, where positive samples indicate ID and vice versa. 





     \subsection{Real-world Medical Dataset Experiments}

In this section, we provide the OOD detection and offline RL results on the real-world medical dataset.
 All evaluations use the Transformer Digital Twin (TDT) simulation \cite{tumay2025guardian} with 6-hour rollouts, matching the dataset’s deterministic 6-hour episode length. We utilize average reward, weaning score (WS), and action change penalty (ACP) for evaluation, where reward signifies the patient's stability based on physiological features; weaning score quantifies the successful decrease in the P-level after observing patient stability as follows \cite{tumay2025guardian}:


\begin{figure*}[htb]
\centering
\begin{minipage}[t]{0.54\linewidth}

\begin{equation}
    \text{WS} = \frac{\sum_{i=1}^{T} \mathbb{I}(\texttt{Is\_Stable}(i),1) \cdot \texttt{Weaned}(i)}
    {\sum_{i=1}^{T}  \mathbb{I}(\texttt{Is\_Stable}(i),1)}
    \vspace{-0.2em}
\label{eq:WS}
\end{equation}
where $T$ is episode length. See Appendix \ref{sec:abiomed} for \texttt{Is\_Stable}, and \texttt{Weaned}, and dataset details.
Action change penalty (ACP) \cite{ogsrl, tumay2025guardian} accumulates the magnitude of P-level changes over an episode of length $T$ as
\vspace{-0.5em}
\begin{equation*}
    \text{ACP =} \sum^{T}_{i=1}||a_{i-1}-a_{i}||_2, \text{ if }  ||a_{i-1}-a_{i}||_2 >2.
\end{equation*}

\end{minipage}
\hfill
\begin{minipage}[t]{0.44\linewidth}
\vspace{0.5em}

     \centering
   
    \includegraphics[width=0.95\linewidth]{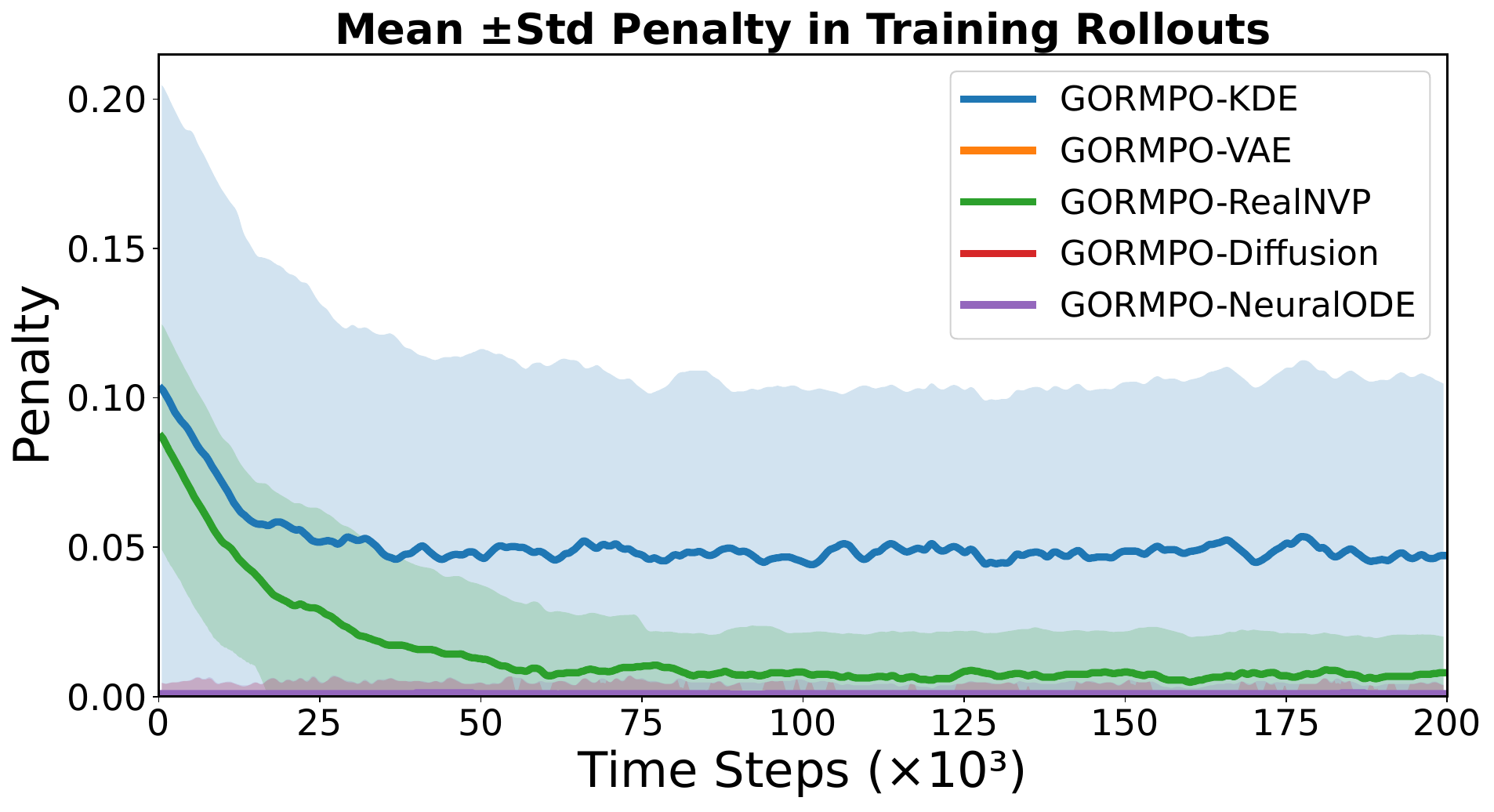}
       \vspace{-0.5em}
    \caption{Penalty during training rollouts for 200000 steps in our medical dataset. GORMPO-Diffusion and VAE overlap at 0.
    }
    \label{fig:penalty_medical}
\end{minipage}
\end{figure*}

\begin{wrapfigure}{r}{0.44\linewidth}
\vspace{-0.5em}
      \centering
    \includegraphics[width=1\linewidth]{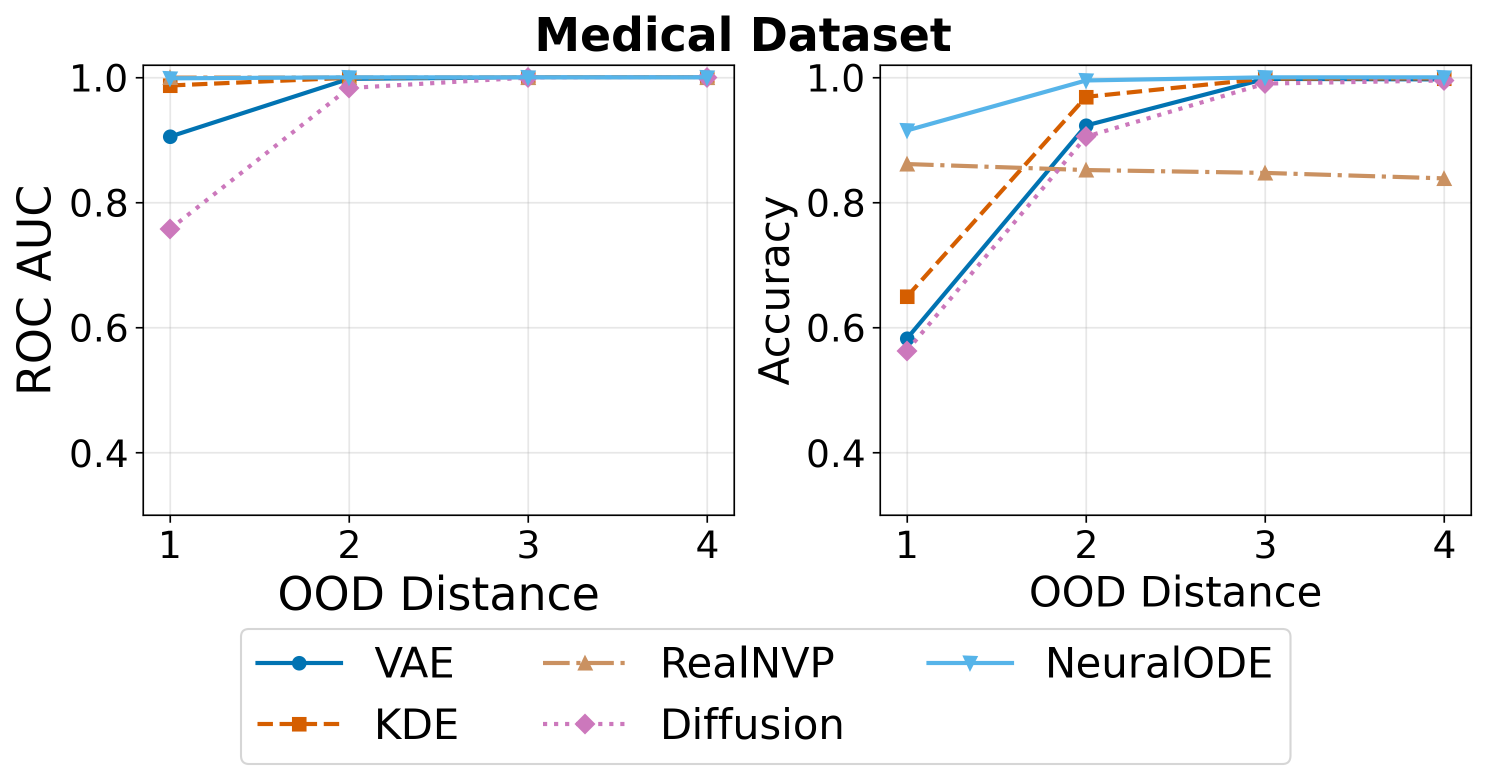}
       \vspace{-1.4em}
    \caption{OOD detection performance of density estimation models trained on the real-world medical dataset in terms of ROC AUC, and accuracy.  OOD distance denotes the mean shift $\mu$ in $\mathcal{N}(\mu, 0.1)$, with larger $\mu$ making OOD detection easier.}
    \label{fig:tnr_tpr}

     \vspace{-0.5em}
\end{wrapfigure}


\paragraph{OOD Detection Results.}
On our proprietary medical dataset, NeuralODE depicts the best ROC AUC followed by the highest accuracy progression, as the OOD detection task gets harder. Other models also perform close to NeuralODE, indicating that this dataset is well-fitted to our sparsity problem.  
This suggests that the valid medical states lie on a sharp, low-dimensional manifold. The physiological rules governing patient stability create a clear separation from random or unsafe regions, simplifying the density estimation task.
RealNVP’s accuracy stays roughly constant even as ROC AUC improves because the density threshold is chosen on a validation set that labels 1\% of ID samples as OOD. This setup introduces an unavoidable error floor in accuracy.
  
\begin{figure*}[!htb]
\vspace{-0.5em}
    \centering
    \includegraphics[width=1\linewidth]{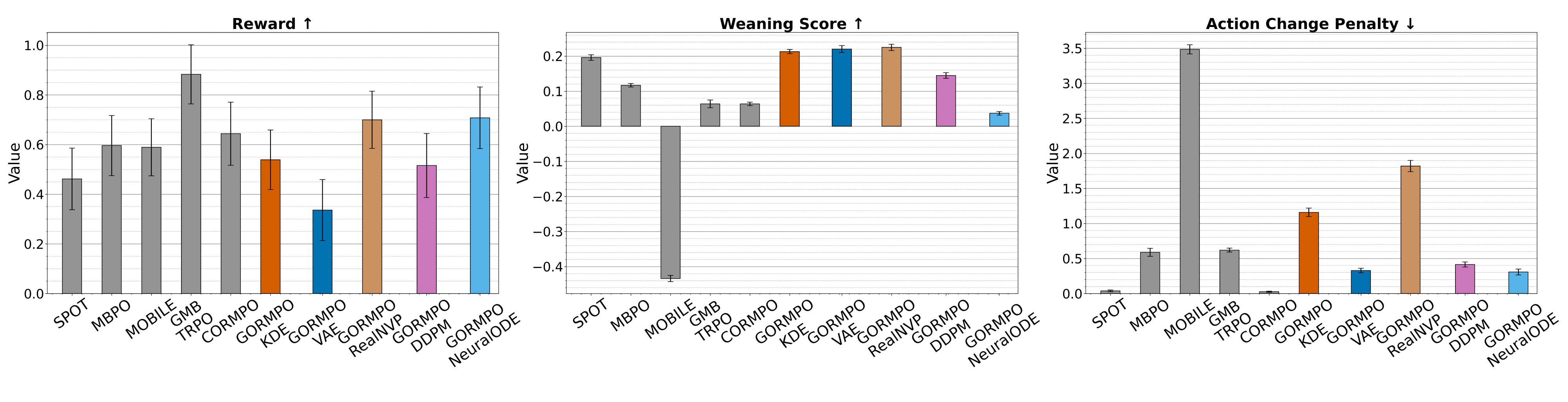}
    \vspace{-2em}
    \caption{Comparison of GORMPO against offline RL baselines on our proprietary medical dataset in reward ($\uparrow$), WS ($\uparrow$), and ACP ($\downarrow$) on 10 seeds. $\uparrow$: higher the better; $\downarrow$: lower the better.
    } 
    \vspace{-1em}
    \label{fig:abiomed_barplots}
   
\end{figure*}

\begin{figure*}[!t]
    \centering
    \includegraphics[width=1\linewidth]{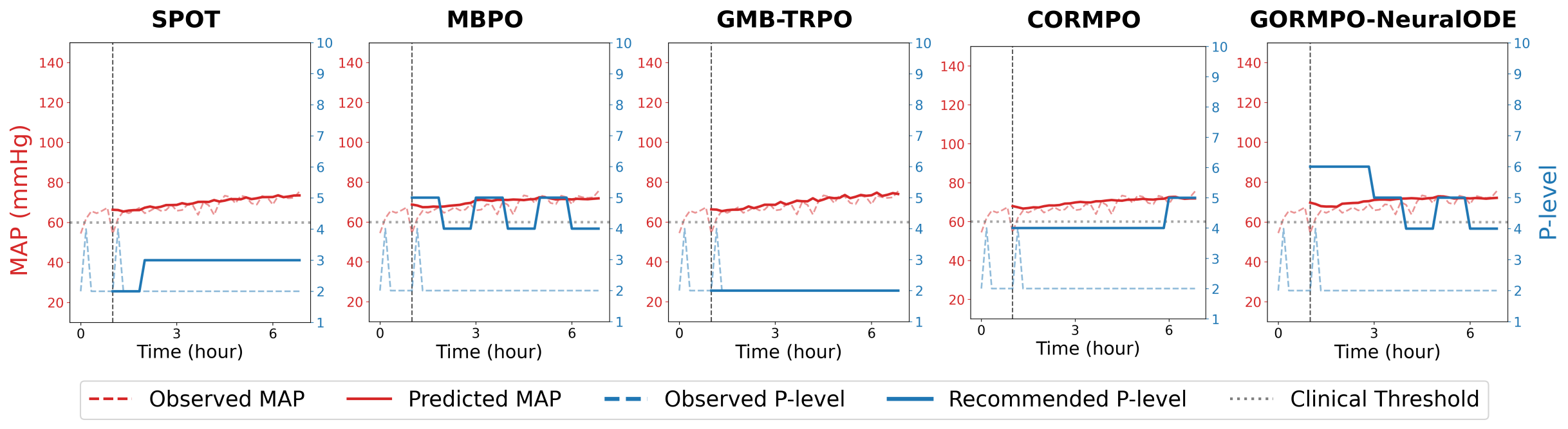}
    \caption{On our proprietary medical dataset, we compare the recommended GORMPO P-levels (solid blue) against offline RL baselines in 6-hour mean arterial pressure (MAP) rollouts (solid red), which is a significant variable for hemodynamic stability, on an OOD expert policy (low reward - low P-level). 
    Since TDT predicts stable hemodynamics above the MAP clinical threshold, an ideal weaning policy is to decrease the P-level every 1 hour, which GORMPO-NeuralODE follows better than the baselines.
    }
    \label{fig:abiomed_samples}
    \vspace{-0.5em}
\end{figure*}


\begin{figure*}[!t]
    \centering
    \includegraphics[width=0.87\linewidth]{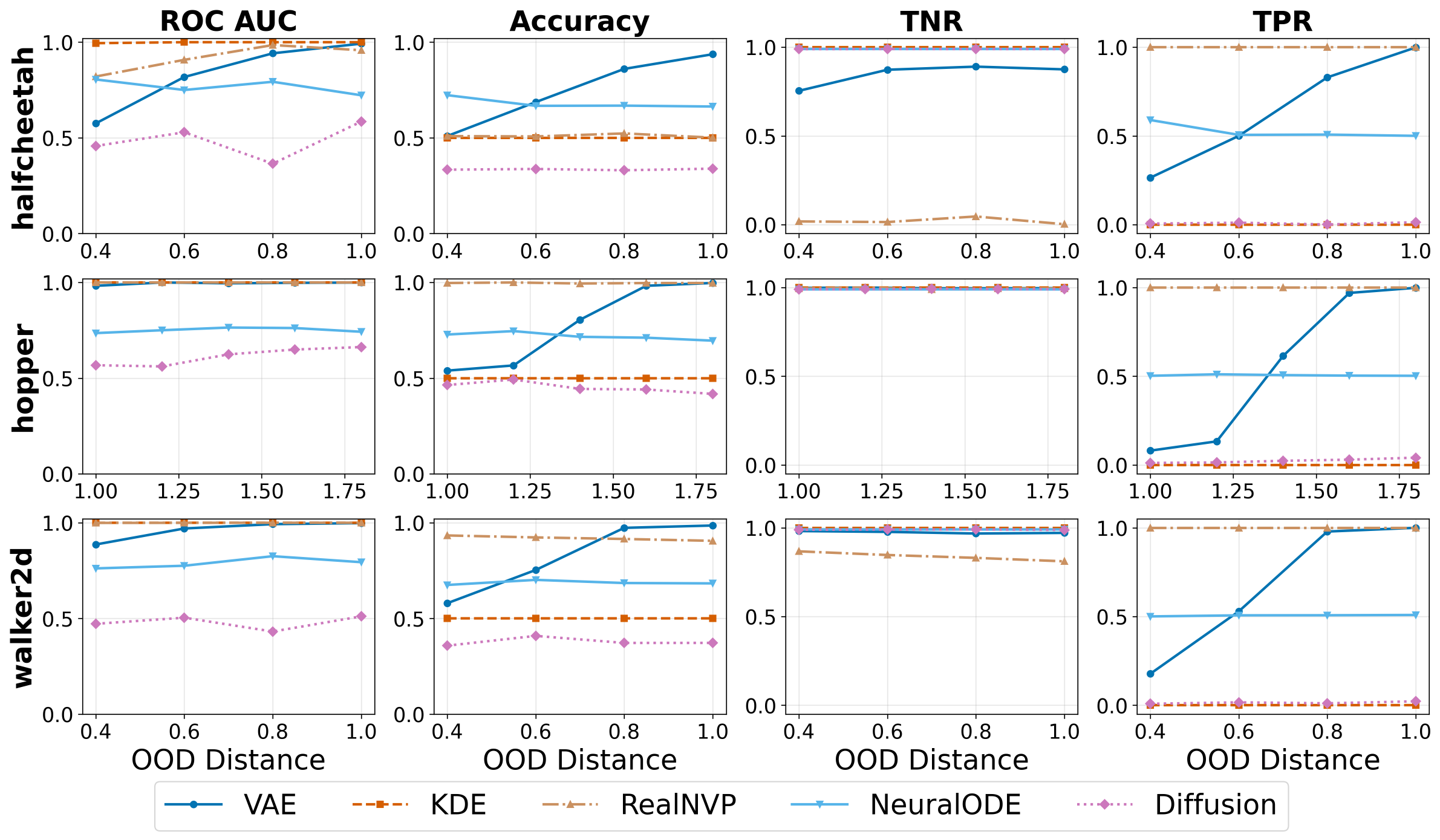}
    \caption{OOD detection performance of KDE, VAE, RealNVP, diffusion, and NeuralODE models trained on sparse halfcheetah, hopper, and walker2d in ROC AUC, accuracy, TNR, and TPR.  OOD distance denotes the mean shift $\mu$ in $\mathcal{N}(\mu, 0.1)$, with larger $\mu$ making OOD detection easier.
    }
    \label{fig:ood_experiments}
    \vspace{-1.5em}
\end{figure*}

\paragraph{RL Results.}

We further show GORMPO's performance on our medical dataset in Figure~\ref{fig:abiomed_barplots}  and~\ref{fig:abiomed_samples}, which reveal complementary strengths across methods.
GORMPO-NeuralODE achieves the highest reward and lowest ACP among the GORMPO variants, while VAE and DDPM result in suboptimal patient stability (reward) since they fail at penalizing the policy enough to avoid unsafe actions (see Figure \ref{fig:penalty_medical}).
GORMPO-RealNVP achieves the highest mean WS among all evaluated models, while GORMPO-RealNVP and NeuralODE outperform SPOT, MBPO, MOBILE, and CORMPO in reward.  GMB-TRPO has a better reward than GORMPO-NeuralODE, which has a substantially lower ACP, but both show low WS. So, GMB-TRPO achieves the highest patient stability through more frequent P-level changes, but largely neglects the weaning goal. CORMPO, meanwhile, attains the lowest ACP overall but a low WS.  Hence, GORMPO-NeuralODE maintains patient stability with a conservative policy, whereas CORMPO is even more conservative at a lower reward. While CORMPO relies on clinical reward shaping specific to this medical task, GORMPO's core formulation requires no domain-specific shaping and remains applicable to non-medical tasks as well.
All in all, no single method dominates across all three metrics simultaneously, underscoring the importance of evaluating all metrics jointly. Please see Appendix~\ref{sec:more_rl} for results tables and more policy visualizations.



\subsection{Sparse D4RL Benchmark Experiments}
In this section, we provide the OOD detection and offline RL results of GORMPO using both MBPO \cite{janner2019mbpo} and MOBILE \cite{pmlr-v202-sun23q} as base models on the sparse D4RL datasets.
\vspace{-0.8em}

\paragraph{Datasets.}
We show performance on the \textbf{Gym-MuJoCo} tasks of the D4RL \cite{fu2020d4rl} benchmark at the medium-expert level due to its high resemblance to real-world clinical decision dynamics, i.e., a clinician adjusting P-level on MCS without access to patient hemodynamic forecasts. 
We create the sparse datasets by randomly discarding 27.5\%, 22\%, and 27\% of the halfcheetah, hopper, and walker2d samples, respectively, from a designated unsafe area with sparsity in the action-reward space. See Appendix \ref{sec:more_data} for more details.

\vspace{-0.8em}
\paragraph{OOD Detection Results.}
We depict the OOD detection performance in Figure \ref{fig:ood_experiments} with VAE being the best in halfcheetah; and RealNVP in hopper, and walker2d.  
In halfcheetah, VAE performs consistently close to the top performance in all metrics, while other models demonstrate inconsistencies. Since the halfcheetah dataset comprises 2 distinct modes, a variational generative model with a successful threshold selection can display robust performance. 
In hopper and walker2d, RealNVP achieves the strongest OOD detection. Since there are many outliers around the main modes, capturing the heavy tails benefits from an explicit density estimator rather than models based on a lower bound estimation (see Appendix \ref{sec:ood_data}). KDE's performance is near random chance despite high ROC AUC scores due to conservative threshold selection. Diffusion exhibits the poorest ROC AUC, as it assigns nearly indistinguishable likelihoods to ID and OOD samples. NeuralODE ranks in the middle, as its threshold captures only about half of the ID samples while flagging nearly all OOD samples.

\begin{table*}[t]
\centering
\caption{Average normalized reward $\pm$ standard deviation over three seeds
on sparse D4RL medium-expert datasets. GORMPO variants are grouped by their
underlying model-based algorithm. Bold indicates the best density estimator per base model and dataset. m-e-s = medium-expert-sparse.}
\label{tab:reward_sparse}

\tiny
\setlength{\tabcolsep}{2.5pt}
\vspace{-0.5em}

\begin{tabular}{l c @{\hspace{8pt}} l c @{\hspace{8pt}} ccccc}
\toprule
\multirow{2}{*}{\textbf{Dataset}}
& \textbf{Offline RL baseline}
& \multicolumn{2}{c}{\textbf{Base model}}
& \multicolumn{5}{c}{\textbf{GORMPO}} \\
\cmidrule(lr){2-2}
\cmidrule(lr){3-4}
\cmidrule(lr){5-9}
& \textbf{SPOT}
& \textbf{Method}
& \textbf{Reward}
& \textbf{KDE}
& \textbf{VAE}
& \textbf{RealNVP}
& \textbf{DDPM}
& \textbf{NeuralODE} \\
\midrule

\multirow{2}{*}{halfcheetah-m-e-s}
& \multirow{2}{*}{\result{76.4}{3.22}}
& MOBILE & \result{104.0}{1.63}
& \result{103.7}{1.70} & \bestresult{106.0}{0.54} & \result{104.6}{1.18} & \result{103.0}{1.44} & \result{104.8}{0.35} \\
& & MBPO & \result{63.7}{8.23}
& \result{61.7}{15.1} & \bestresult{89.9}{1.50} & \result{86.0}{8.34} & \result{85.7}{4.18} & \result{80.3}{13.7} \\
\midrule

\multirow{2}{*}{hopper-m-e-s}
& \multirow{2}{*}{\result{83.8}{23.3}}
& MOBILE & \result{95.2}{24.5}
& \result{95.2}{30.0} & \bestresult{113.0}{0.53} & \result{95.5}{29.8} & \bestresult{113.0}{0.62} & \result{112.9}{0.62} \\
& & MBPO & \result{4.81}{3.18}
& \result{8.91}{12.89} & \result{2.32}{0.68} & \result{10.0}{9.02} & \bestresult{14.1}{8.66} & \result{5.05}{3.85} \\
\midrule

\multirow{2}{*}{walker2d-m-e-s}
& \multirow{2}{*}{\result{113.5}{0.54}}
& MOBILE & \result{116.0}{2.44}
& \result{117.0}{1.32} & \result{117.0}{4.34} & \bestresult{118.0}{2.66} & \result{117.0}{4.45} & \result{49.2}{40.72} \\
& & MBPO & \result{3.19}{1.21}
& \result{6.31}{1.26} & \result{8.20}{6.04} & \bestresult{10.1}{5.34} & \result{5.23}{0.53} & \result{3.44}{2.19} \\
\bottomrule
\end{tabular}

\vspace{-2em}
\end{table*}

\vspace{-0.5em}
\paragraph{RL Results.} We demonstrate offline RL results on the sparse medium-expert in Table \ref{tab:reward_sparse}. 
In halfcheetah, MBPO-based GORMPO outperforms MBPO by 41\% while remaining competitive with SOTA baselines, highlighting the density-based guardian's benefit for regularizing the OOD transition rollouts in MBPO.
Walker2d and hopper datasets exhibit high uncertainty in the terminal signal, unlike halfcheetah, which has uniform episode lengths (see Appendix \ref{sec:more_data}). In this case, MBPO in walker2d and hopper fails, and MBPO-based GORMPO's improvement on MBPO is insufficient to outperform SOTA methods. Without a direct uncertainty penalty on the critic network like MOBILE, MBPO is sensitive to compounding model error and termination misprediction, where small transition mistakes quickly push model rollouts OOD. Using MOBILE as the base model, our density-based guardian yields an 18.7\% improvement in hopper, demonstrating that it further strengthens SOTA uncertainty-penalized methods, especially under hopper's highly uncertain dynamics. MOBILE-based GORMPO-NeuralODE’s high variance on
 walker2d is consistent with walker2d's uncertain termination signal in Figure \ref{fig:terminal} and NeuralODE's fluctuating penalty in Figure \ref{fig:all_penalty}). 
 
\begin{wrapfigure}{r}{0.32\linewidth}
\vspace{0.5em}
\centering
    \centering

    \includegraphics[width=1\linewidth]{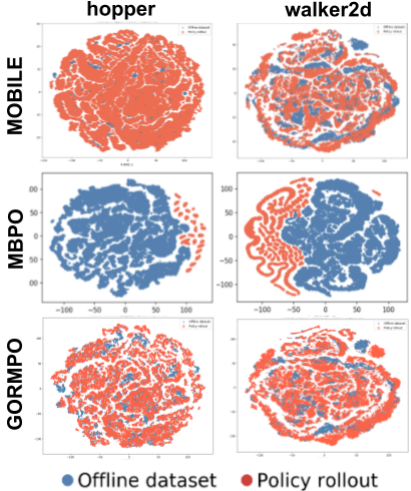}
    \caption{t-SNE projections of equal-sized offline dataset and policy rollout $(s',a)$ pairs.}
    \label{fig:tsne}

     \vspace{-2.5em}
\end{wrapfigure}
\vspace{-0.5em}
In Figure \ref{fig:tsne}, MOBILE covers the data support in hopper and walker2d,
meanwhile MBPO largely reaches OOD regions. The best-performing MOBILE-based GORMPO variants have better data coverage by avoiding OOD regions and overlapping with the dataset projection.
See more MOBILE-based GORMPO results on the medium-level D4RL datasets in Appendix \ref{sec:more_rl}.

\subsection{Does better OOD detection lead to better policies?}
We combine the OOD detection and RL results in this section.  In our proprietary medical dataset, NeuralODE is the best OOD detection model in all metrics and in all OOD distance settings. Aligned with this, GORMPO-NeuralODE is also the best policy in reward and ACP among GORMPO variants. As this dataset exhibits vast sparsities, penalties with exact likelihoods are essential for policy performance. 
In halfcheetah, VAE is consistently satisfactory at each metric with a distinctive threshold selection, which also coincides with the high reward of GORMPO-VAE. Since halfcheetah has a uniform episode length distribution, density estimators with nonzero penalties successfully regularized the OOD rollouts of the dynamics model (see Appendix \ref{sec:tsne}). 
On hopper, RealNVP achieves the best OOD detection performance, yet this advantage does not translate into higher GORMPO returns. Despite weaker OOD detection capabilities inherent in diffusion models (Appendix \ref{appendix:ddpm}), DDPM supports stronger policy learning, possibly because its low likelihood scores saturate the tanh penalty near 1 (Appendix \ref{sec:penalty}). Conservatism keeps the policy close to the training support, as shown in Figure \ref{fig:tsne}.
On the walker2d models, RealNVP is the best model in OOD separation, which overlaps with the RL result, with GORMPO-RealNVP being the best variant. 
Overall, the best GORMPO variants by reward are consistent across MBPO-GORMPO and MOBILE-GORMPO. We find that better OOD detection generally yields better policies, though under highly uncertain dynamics, such as hopper, the most pessimistic density models can be preferable.

%% file: sections/conclusion.tex
\label{sec:conc}
In this paper, we propose GORMPO, a generative OOD-regularized model-based policy optimization algorithm for sequential decision-making in sparse state-action spaces. During policy optimization, we penalize rewards predicted by the dynamics model based on the joint probability of next-state and action pairs with generative density estimator models. Our analysis shows that GORMPO can achieve stronger near-optimality guarantees while explicitly accounting for reward-model error.
Empirically, GORMPO remains jointly competitive across reward and clinical metrics compared to the SOTA clinical offline RL policies, also outperforming the SOTA model-based offline RL methods on sparse offline RL benchmarks.
Due to our plug-and-play framework, 
we benchmark the OOD detection performance of five distinct families of density estimators on our medical dataset and sparse offline RL datasets.
We conclude that the best OOD detection models generally lead to better policies with exceptions for stochastic dynamics benefiting from conservative density penalties. 
A key limitation of our approach is its reliance on a penalization threshold, since distribution skewness can shift this threshold and alter conservatism. 
Also, our medical policies are assessed on a simulated environment which restricts our conclusion of real-life clinical safety. 
Future work will explore model-based RL with threshold-free density penalization tested on datasets with uncertain dynamics.

%% file: sections/appendix.tex
\section{Algorithm Pseudocode}
\label{algo}
\input{alg_sample}

\section{Proofs}
\label{proofs}



\begin{lemma}[Generalized Telescoping Lemma; generalizes Lemma 4.1 in \cite{yu2020mopo}]
\label{lem:telescoping}

Let ${\mathcal{M}}$ and $\hat{{\mathcal{M}}}$ be two MDPs with \emph{different} reward functions $r$ and
$\hat{r}$, and different dynamics $T$ and $\hat{T}$, respectively. Let
$$G^\pi_{\hat{\mathcal{M}}}(s,a) := \mathbb{E}_{s' \sim \hat{T}(s,a)}[V^\pi_\mathcal{M}(s')] - \mathbb{E}_{s' \sim T(s,a)}[V^\pi_\mathcal{M}(s')].$$
Then,
$$\eta_{\hat{\mathcal{M}}}(\pi) - \eta_\mathcal{M}(\pi) = \mathbb{E}_{(s,a) \sim \rho^\pi_{\hat{T}}}[\hat{r}(s,a) - r(s,a)] + \gamma \mathbb{E}_{(s,a) \sim \rho^\pi_{\hat{T}}}[G^\pi_{\hat{\mathcal{M}}}(s,a)]$$
\end{lemma}

\begin{proof}
Define $W_j$ as the return obtained by executing $\pi$ under $(\hat{T},\hat{r})$
for the first $j$ steps and under $(T,r)$ thereafter, so that
$W_0 = \eta_\mathcal{M}(\pi)$ and $W_\infty = \eta_{\hat{\mathcal{M}}}(\pi)$. The processes
generating $W_j$ and $W_{j+1}$ share an identical prefix through step $j-1$ and
diverge only at step $j$: $W_j$ collects $r(s_j,a_j)$ and transitions via
$s_{j+1} \sim T(s_j,a_j)$, whereas $W_{j+1}$ collects $\hat{r}(s_j,a_j)$ and
transitions via $s_{j+1} \sim \hat{T}(s_j,a_j)$; both are followed by
$V^\pi_\mathcal{M}(s_{j+1})$. Hence
$$W_{j+1} - W_j = \gamma^{j} \mathbb{E}_{\rho_j}[(\hat{r}-r)(s_j,a_j)] + \gamma^{j+1} \mathbb{E}_{\rho_j}[G^\pi_{\hat{\mathcal{M}}}(s_j,a_j)],$$
where $\rho_j$ denotes the time-$j$ occupancy under $\hat{T}$. Summing over
$j \ge 0$ and using $\sum_j \gamma^j \mathbb{E}_{\rho_j}[f] = \mathbb{E}_{\rho^\pi_{\hat{T}}}[f]$
yields the claim. Setting $\hat{r} = r$ recovers Lemma 4.1 of \cite{yu2020mopo}.
\end{proof}

\noindent

\begin{lemma}[Function Class Bound]
\label{lem:function_class}
If $\mathcal{F}$ is a set of functions mapping $\mathcal{S}$ to $\mathbb{R}$ with $\|\cdot\|_\infty \leq 1$, and $V^\pi \in c\mathcal{F}$ where $c = r_{\max}/(1-\gamma)$, then
$$|G^\pi_{\hat{\mathcal{M}}}(s,a)| \leq c \cdot d_{\mathcal{F}}(\hat{T}(s,a), T(s,a))$$
where $d_{\mathcal{F}}(\hat{T}(s,a), T(s,a)) = \sup_{f \in \mathcal{F}} |\mathbb{E}_{s' \sim \hat{T}(s,a)}[f(s')] - \mathbb{E}_{s' \sim T(s,a)}[f(s')]|$.
\end{lemma}

\begin{lemma}[Telescoping Lemma \cite{tumay2025guardian}: Model Approximation Bound for Transition Functions ]

Let $\pi$ be any feasible solution and $\hat{T}_\phi$ be the Transformer-learned transition function with parameters $\phi$. Assume the following:
\begin{itemize}
    \item Dataset $\mathcal{D} = \{(s_i, a_i, r_i, s'_i)\}_{i=1}^N$ with $N$ samples
    \item Transformer with $L$ layers, dimension $d_{model}$, and $H_{attn}$ attention heads
    \item Lipschitz continuous true transition $T$ with constant $L_T$
\end{itemize}

Then, with probability at least $1 - 2\beta - 4\delta$, the following holds:

$$\left|V_{\psi, \hat{T}_\phi}^{\pi}(\rho_0) - V_{\psi, T}^{\pi}(\rho_0)\right| \leq \epsilon_H + \epsilon_{trans}$$

where:

$$\epsilon_H := \frac{\gamma^{H+1}(2-\gamma)\psi_{max}}{(1-\gamma)^2}$$

$$\epsilon_{trans} := \frac{\psi_{max}(\gamma - \gamma^{H+2})}{(1-\gamma)^2} \left( \epsilon_{approx} + \epsilon_{gen} \right)$$

with:

$$\epsilon_{approx} := C_{trans} \cdot \min\left\{\frac{1}{L \cdot d_{model}}, \frac{1}{H_{attn} \cdot N_{ctx}}\right\}$$

$$\epsilon_{gen} := L_T \sqrt{\frac{d \log(1/\delta) + \log N}{N}} + \mathcal{O}\left(\frac{1}{\sqrt{N_{eff}}}\right)$$

Assumption \ref{ass2} is supported by the above derivation of model approximation error.  In regions with high estimated density $p_\theta(s',a)$, the local sample density is high, yielding large $N_{\text{eff}}$ and small error. Conversely, as $p_{\theta}(s',a) \rightarrow 0$, the effective sample size diminishes, causing the error to grow. The linear relationship $\tau - p_{\theta}(s,a)$ captures this first-order dependence between local data density and model error, while $\epsilon_{\text{approx}} = C_{\text{trans}} \cdot \min\{1/(L \cdot d_{\text{model}}), 1/(H_{\text{attn}} \cdot N_{\text{ctx}})\}$ represents the transformer's irreducible approximation error determined by its architecture (depth $L$, dimension $d_{\text{model}}$, and attention heads $H_{\text{attn}}$). The constant $C_{\hat{T}}$ encapsulates the Lipschitz constant $L_T$ of the true dynamics and the dimensionality-dependent factors.
\end{lemma}

\begin{remark} Assumption \ref{ass3} is practically satisfied when the density estimator is expressive enough and trained to convergence on $\mathcal{D}$. One cannot directly measure $\epsilon_{density}$ on our datasets since the ground truth probability distribution is unknown. However, we provide a negative log-likelihood evaluation on the test set of sparse D4RL datasets in Table \ref{tab:nll_results} in Appendix \ref{sec:more_rl}.
Theoretical verification of Assumption \ref{ass3} is not possible.
\end{remark}

\begin{lemma}[Density Estimation Error Propagation]
\label{lem:density_error}
Under Assumption \ref{ass3}, the difference between the true regularizer $u(s',a) = \tanh(\max(\tau - \log p(s',a), 0))$ and the estimated regularizer $\hat{u}(\hat s',a) = \tanh(\max(\tau - \log p_\theta(s',a), 0))$ satisfies:
$$|\hat{u}(s',a) - u(s',a)| \leq \epsilon_{\text{density}}$$
for all $(s',a) \in \mathcal{S} \times \mathcal{A}$.
\end{lemma}

\begin{proof}
Since $\tanh(\cdot)$ is 1-Lipschitz and $\max(\cdot, 0)$ is 1-Lipschitz:
\begin{align}
|\hat{u}(s',a) - u(s',a)| &= |\tanh(\max(\tau - \log p_\theta(s',a), 0)) - \tanh(\max(\tau - \log p(s',a), 0))| \\
&\leq |\max(\tau - \log p_\theta(s',a), 0) - \max(\tau - \log p(s',a), 0)| \\
&\leq |(\tau - \log p_\theta(s',a)) - (\tau - \log p(s',a))| \\
&= |\log p(s',a) - \log p_\theta(s',a)| \\
&\leq \epsilon_{\text{density}}
\end{align}
where the last inequality follows from Assumption \ref{ass3}.
\end{proof}

\subsection{Proof of Theorem \ref{thm:1}: Conservative Value Bound with Density and Reward Error}

\begin{proof}
Starting from Lemma~\ref{lem:telescoping} and applying
Assumption~\ref{ass4:reward}, with the unnormalized occupancy measure
$\rho^\pi_{\hat T}$ of total mass $1/(1-\gamma)$:
\begin{align}
\eta_\mathcal{M}(\pi)
&= \eta_{\hat{\mathcal{M}}}(\pi)
-\mathbb{E}_{(s,a)\sim\rho^\pi_{\hat T}}[(\hat r-r)(s,a)]
-\gamma\mathbb{E}_{(s,a)\sim\rho^\pi_{\hat T}}
[G^\pi_{\hat{\mathcal{M}}}(s,a)]
\\
&\geq \eta_{\hat{\mathcal{M}}}(\pi)
-\frac{\epsilon_{\text{reward}}}{1-\gamma}
-\gamma\mathbb{E}_{(s,a)\sim\rho^\pi_{\hat T}}
[|G^\pi_{\hat{\mathcal{M}}}(s,a)|].
\end{align}

By Assumption~\ref{ass2} and Lemma~\ref{lem:function_class},
with the nonnegative true regularizer $u(s',a)$ based on
the true density $p(s',a)$:
\[
|G^\pi_{\hat{\mathcal{M}}}(s,a)|
\leq c\,d_{\mathcal F}(\hat T(s,a),T(s,a))
\leq c C_{\hat T}
\mathbb{E}_{s'\sim\hat T(s,a)}[u(s',a)]
+c\epsilon_{\text{approx}}.
\]

We now relate the true regularizer $u(s',a)$ to the estimated
regularizer $\hat u(s',a)$ based on the learned density
$p_\theta(s',a)$, using the same penalty transformation and
threshold $\tau$.
Note that
$\eta_{\hat{\mathcal M}}(\pi)
=\mathbb{E}_{(s,a)\sim\rho^\pi_{\hat T}}[\hat r(s,a)]$.
We add and subtract the penalty term based on the estimated
regularizer:
\begin{align}
\eta_\mathcal{M}(\pi)
&\geq
\mathbb{E}_{(s,a)\sim\rho^\pi_{\hat T}}[\hat r(s,a)]
-\gamma c C_{\hat T}
\mathbb{E}_{(s,a)\sim\rho^\pi_{\hat T}}
\!\Big[\mathbb{E}_{s'\sim\hat T(s,a)}[u(s',a)]\Big]
\notag\\
&\quad
-\frac{\gamma c\,\epsilon_{\text{approx}}}{1-\gamma}
-\frac{\epsilon_{\text{reward}}}{1-\gamma}
\\
&=
\mathbb{E}_{(s,a)\sim\rho^\pi_{\hat T}}
\!\Big[\hat r(s,a)
-\lambda\mathbb{E}_{s'\sim\hat T(s,a)}[\hat u(s',a)]\Big]
\notag\\
&\quad
+\lambda\mathbb{E}_{(s,a)\sim\rho^\pi_{\hat T}}
\!\Big[\mathbb{E}_{s'\sim\hat T(s,a)}[\hat u(s',a)]\Big]
\notag\\
&\quad
-\gamma c C_{\hat T}
\mathbb{E}_{(s,a)\sim\rho^\pi_{\hat T}}
\!\Big[\mathbb{E}_{s'\sim\hat T(s,a)}[u(s',a)]\Big]
\notag\\
&\quad
-\frac{\gamma c\,\epsilon_{\text{approx}}}{1-\gamma}
-\frac{\epsilon_{\text{reward}}}{1-\gamma}.
\label{eq:your_bound}
\end{align}

By Lemma~\ref{lem:density_error},
$|u(s',a)-\hat u(s',a)|\leq\epsilon_{\text{density}}$
for all $(s',a)$. Therefore, using the total mass of
$\rho^\pi_{\hat T}$:
\begin{align}
&\mathbb{E}_{(s,a)\sim\rho^\pi_{\hat T}}
\!\Big[\mathbb{E}_{s'\sim\hat T(s,a)}[u(s',a)]\Big]
\notag\\
&\qquad\leq
\mathbb{E}_{(s,a)\sim\rho^\pi_{\hat T}}
\!\Big[\mathbb{E}_{s'\sim\hat T(s,a)}[\hat u(s',a)]\Big]
+\frac{\epsilon_{\text{density}}}{1-\gamma}.
\end{align}

Define
\[
\tilde r(s,a)
=\hat r(s,a)
-\lambda\mathbb{E}_{s'\sim\hat T(s,a)}[\hat u(s',a)].
\]
Substituting into our bound:
\begin{align}
\eta_\mathcal{M}(\pi)
&\geq
\mathbb{E}_{(s,a)\sim\rho^\pi_{\hat T}}[\tilde r(s,a)]
+\lambda\mathbb{E}_{(s,a)\sim\rho^\pi_{\hat T}}
\!\Big[\mathbb{E}_{s'\sim\hat T(s,a)}[\hat u(s',a)]\Big]
\notag\\
&\quad
-\gamma c C_{\hat T}
\left(
\mathbb{E}_{(s,a)\sim\rho^\pi_{\hat T}}
\!\Big[\mathbb{E}_{s'\sim\hat T(s,a)}[\hat u(s',a)]\Big]
+\frac{\epsilon_{\text{density}}}{1-\gamma}
\right)
\notag\\
&\quad
-\frac{\gamma c\,\epsilon_{\text{approx}}}{1-\gamma}
-\frac{\epsilon_{\text{reward}}}{1-\gamma}
\\
&=
\mathbb{E}_{(s,a)\sim\rho^\pi_{\hat T}}[\tilde r(s,a)]
+(\lambda-\gamma c C_{\hat T})
\mathbb{E}_{(s,a)\sim\rho^\pi_{\hat T}}
\!\Big[\mathbb{E}_{s'\sim\hat T(s,a)}[\hat u(s',a)]\Big]
\notag\\
&\quad
-\frac{\gamma c C_{\hat T}\epsilon_{\text{density}}}{1-\gamma}
-\frac{\gamma c\,\epsilon_{\text{approx}}}{1-\gamma}
-\frac{\epsilon_{\text{reward}}}{1-\gamma}.
\end{align}

Setting $\lambda=\gamma c C_{\hat T}$, the middle term vanishes:
\begin{align}
\eta_\mathcal{M}(\pi)
&\geq
\mathbb{E}_{(s,a)\sim\rho^\pi_{\hat T}}[\tilde r(s,a)]
\notag\\
&\quad
-\frac{
\gamma c C_{\hat T}\epsilon_{\text{density}}
+\gamma c\,\epsilon_{\text{approx}}
+\epsilon_{\text{reward}}
}{1-\gamma}.
\end{align}

Since $\tilde{\mathcal M}$ shares the dynamics $\hat T$,
initial distribution $\mu_0$, and discount factor $\gamma$
with $\hat{\mathcal M}$,
\[
\mathbb{E}_{(s,a)\sim\rho^\pi_{\hat T}}[\tilde r(s,a)]
=\eta_{\tilde{\mathcal M}}(\pi)
\]
using the unnormalized occupancy measure. Thus,
\[
\eta_\mathcal{M}(\pi)
\geq
\eta_{\tilde{\mathcal M}}(\pi)
-\frac{
\gamma c\,\epsilon_{\text{approx}}
+\lambda\epsilon_{\text{density}}
+\epsilon_{\text{reward}}
}{1-\gamma},
\]
where
$\tilde{\mathcal M}
=(\mathcal S,\mathcal A,\hat T,\tilde r,\mu_0,\gamma)$
is the density-regularized MDP.
\end{proof}

\subsection{Proof of Theorem \ref{thm:thm2}:
Optimality Gap with Density and Reward Error}

\begin{proof}
We establish a lower bound on $\eta_{\mathcal M}(\hat\pi)$
by comparing it with policies in a constrained class.

For brevity, write
\[
J_u(\pi)
:=
\mathbb{E}_{(s,a)\sim\rho^\pi_{\hat T}}
\!\Big[
\mathbb{E}_{s'\sim\hat T(s,a)}[u(s',a)]
\Big],
\qquad
J_{\hat u}(\pi)
:=
\mathbb{E}_{(s,a)\sim\rho^\pi_{\hat T}}
\!\Big[
\mathbb{E}_{s'\sim\hat T(s,a)}[\hat u(s',a)]
\Big].
\]
Here, $\rho^\pi_{\hat T}$ is the unnormalized discounted
occupancy measure, with total mass $1/(1-\gamma)$.

From Lemma~\ref{lem:telescoping},
Assumption~\ref{ass4:reward}, and
Lemma~\ref{lem:function_class}, for any policy $\pi$:
\begin{align}
|\eta_{\hat{\mathcal M}}(\pi)-\eta_{\mathcal M}(\pi)|
&\leq
\frac{\epsilon_{\mathrm{reward}}}{1-\gamma}
+\gamma
\mathbb{E}_{(s,a)\sim\rho^\pi_{\hat T}}
[|G^\pi_{\hat{\mathcal M}}(s,a)|]
\\
&\leq
\frac{\epsilon_{\mathrm{reward}}}{1-\gamma}
+\gamma c
\mathbb{E}_{(s,a)\sim\rho^\pi_{\hat T}}
[d_{\mathcal F}(\hat T(s,a),T(s,a))].
\end{align}

By Assumption~\ref{ass2}, with the nonnegative true
regularizer $u(s',a)$:
\begin{equation}
|\eta_{\hat{\mathcal M}}(\pi)-\eta_{\mathcal M}(\pi)|
\leq
\gamma c C_{\hat T}J_u(\pi)
+\frac{
\gamma c\,\epsilon_{\mathrm{approx}}
+\epsilon_{\mathrm{reward}}
}{1-\gamma}.
\label{eq:two_sided}
\end{equation}

Applying \eqref{eq:two_sided} to $\hat\pi$ gives
\begin{equation}
\eta_{\mathcal M}(\hat\pi)
\geq
\eta_{\hat{\mathcal M}}(\hat\pi)
-\gamma c C_{\hat T}J_u(\hat\pi)
-\frac{
\gamma c\,\epsilon_{\mathrm{approx}}
+\epsilon_{\mathrm{reward}}
}{1-\gamma}.
\label{eq:pi_hat_lower}
\end{equation}

By optimality of $\hat\pi$ in the density-regularized MDP, $\tilde{\mathcal{M}}$,
for any policy $\pi$:
\[
\eta_{\hat{\mathcal M}}(\hat\pi)
-\lambda J_{\hat u}(\hat\pi)
\geq
\eta_{\hat{\mathcal M}}(\pi)
-\lambda J_{\hat u}(\pi).
\]
Rearranging:
\begin{equation}
\eta_{\hat{\mathcal M}}(\hat\pi)
\geq
\eta_{\hat{\mathcal M}}(\pi)
-\lambda J_{\hat u}(\pi)
+\lambda J_{\hat u}(\hat\pi).
\label{eq:optimality}
\end{equation}

By Lemma~\ref{lem:density_error},
$|\hat u(s',a)-u(s',a)|
\leq\epsilon_{\mathrm{density}}$
for all $(s',a)$.
Integrating against the model transition distribution and
the unnormalized occupancy measure yields
\begin{align}
J_{\hat u}(\hat\pi)
&\geq J_u(\hat\pi)
-\frac{\epsilon_{\mathrm{density}}}{1-\gamma},
\label{eq:pi_hat_u_lower}\\
J_{\hat u}(\pi)
&\leq J_u(\pi)
+\frac{\epsilon_{\mathrm{density}}}{1-\gamma}.
\label{eq:pi_u_upper}
\end{align}

Using \eqref{eq:pi_hat_u_lower} and
\eqref{eq:pi_u_upper} in \eqref{eq:optimality}:
\begin{align}
\eta_{\hat{\mathcal M}}(\hat\pi)
&\geq
\eta_{\hat{\mathcal M}}(\pi)
-\lambda\left(
J_u(\pi)+\frac{\epsilon_{\mathrm{density}}}{1-\gamma}
\right)
\notag\\
&\quad
+\lambda\left(
J_u(\hat\pi)-\frac{\epsilon_{\mathrm{density}}}{1-\gamma}
\right)
\\
&=
\eta_{\hat{\mathcal M}}(\pi)
-\lambda J_u(\pi)
+\lambda J_u(\hat\pi)
-\frac{2\lambda\epsilon_{\mathrm{density}}}{1-\gamma}.
\label{eq:with_true_u}
\end{align}

Substituting \eqref{eq:with_true_u} into
\eqref{eq:pi_hat_lower}:
\begin{align}
\eta_{\mathcal M}(\hat\pi)
&\geq
\eta_{\hat{\mathcal M}}(\pi)
-\lambda J_u(\pi)
+(\lambda-\gamma c C_{\hat T})J_u(\hat\pi)
\notag\\
&\quad
-\frac{
2\lambda\epsilon_{\mathrm{density}}
+\gamma c\,\epsilon_{\mathrm{approx}}
+\epsilon_{\mathrm{reward}}
}{1-\gamma}.
\end{align}

Setting $\lambda=\gamma c C_{\hat T}$, the terms involving
$J_u(\hat\pi)$ cancel:
\begin{equation}
\eta_{\mathcal M}(\hat\pi)
\geq
\eta_{\hat{\mathcal M}}(\pi)
-\lambda J_u(\pi)
-\frac{
2\lambda\epsilon_{\mathrm{density}}
+\gamma c\,\epsilon_{\mathrm{approx}}
+\epsilon_{\mathrm{reward}}
}{1-\gamma}.
\label{eq:before_conversion}
\end{equation}

We now relate the comparator's model return to its true
return. From \eqref{eq:two_sided}:
\[
\eta_{\hat{\mathcal M}}(\pi)
\geq
\eta_{\mathcal M}(\pi)
-\gamma c C_{\hat T}J_u(\pi)
-\frac{
\gamma c\,\epsilon_{\mathrm{approx}}
+\epsilon_{\mathrm{reward}}
}{1-\gamma}.
\]
Substituting into \eqref{eq:before_conversion} and using
$\lambda=\gamma c C_{\hat T}$ gives
\begin{equation}
\eta_{\mathcal M}(\hat\pi)
\geq
\eta_{\mathcal M}(\pi)
-2\lambda J_u(\pi)
-\frac{
2\lambda\epsilon_{\mathrm{density}}
+2\gamma c\,\epsilon_{\mathrm{approx}}
+2\epsilon_{\mathrm{reward}}
}{1-\gamma}.
\label{eq:comparator_bound}
\end{equation}

For any policy $\pi$ satisfying
\[
J_u(\pi)\leq
\delta+\frac{\epsilon_{\mathrm{density}}}{1-\gamma},
\]
equation~\eqref{eq:comparator_bound} implies
\begin{align}
\eta_{\mathcal M}(\hat\pi)
&\geq
\eta_{\mathcal M}(\pi)
-2\lambda\left(
\delta+\frac{\epsilon_{\mathrm{density}}}{1-\gamma}
\right)
\notag\\
&\quad
-\frac{
2\lambda\epsilon_{\mathrm{density}}
+2\gamma c\,\epsilon_{\mathrm{approx}}
+2\epsilon_{\mathrm{reward}}
}{1-\gamma}
\\
&=
\eta_{\mathcal M}(\pi)
-2\lambda\delta
-\frac{
4\lambda\epsilon_{\mathrm{density}}
+2\gamma c\,\epsilon_{\mathrm{approx}}
+2\epsilon_{\mathrm{reward}}
}{1-\gamma}.
\end{align}

Since this holds for every policy in the comparison class,
taking the supremum yields
\[
\eta_{\mathcal M}(\hat\pi)
\geq
\sup_{\pi:\,J_u(\pi)\leq
\delta+\frac{\epsilon_{\mathrm{density}}}{1-\gamma}}
\eta_{\mathcal M}(\pi)
-2\lambda\delta
-\frac{
4\lambda\epsilon_{\mathrm{density}}
+2\gamma c\,\epsilon_{\mathrm{approx}}
+2\epsilon_{\mathrm{reward}}
}{1-\gamma}.
\]
\end{proof}

\clearpage
\section{Ablations}
\label{sec:ablations}

\begin{figure*}[!htp]
    \centering
    \begin{subfigure}{\linewidth}
        \includegraphics[width=\linewidth]{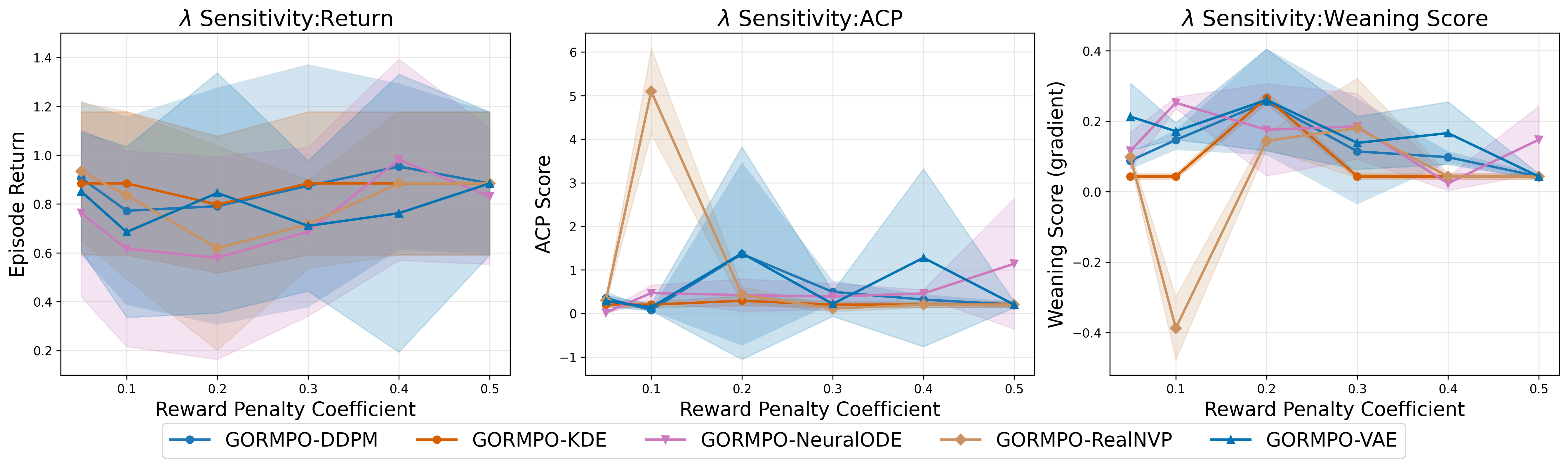}
        \caption{\textbf{Medical dataset.} All scores are stable except for $\lambda=0.1$ in ACP and WS.}
        \label{fig:lambda_Medical}
    \end{subfigure}
    \begin{subfigure}{\linewidth}
        \includegraphics[width=\linewidth]{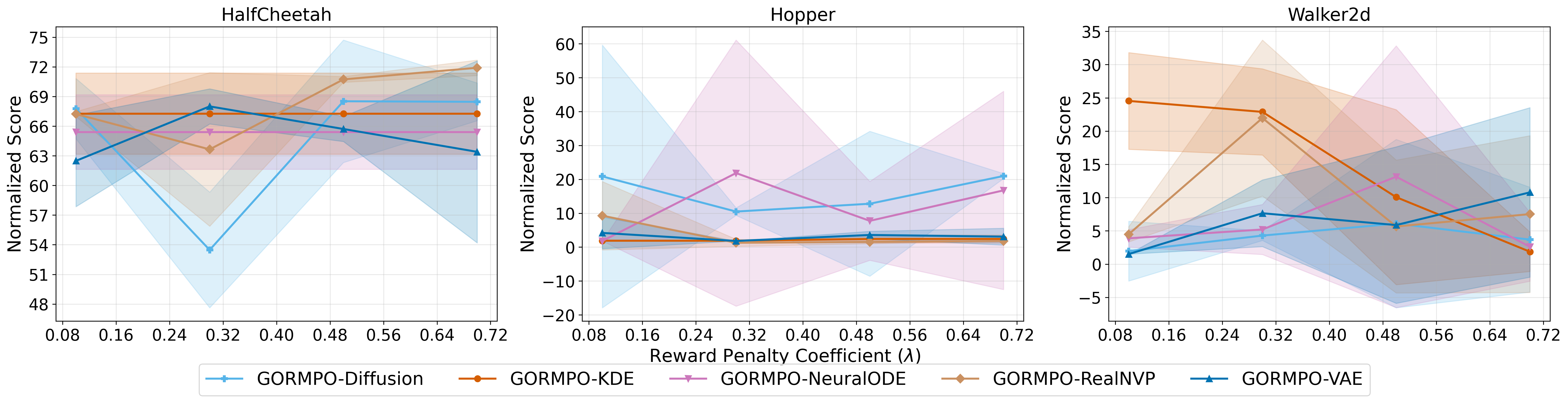}
        \caption{\textbf{Sparse D4RL datasets.}}
        \label{fig:lambda_d4rl}
    \end{subfigure}
    \caption{$\lambda$ sensitivity plots averaged over 2 seeds.}
    \label{fig:lambda_sensitivity}
\end{figure*}

\begin{figure*}[!htp]
    \centering
    \includegraphics[width=1\linewidth]{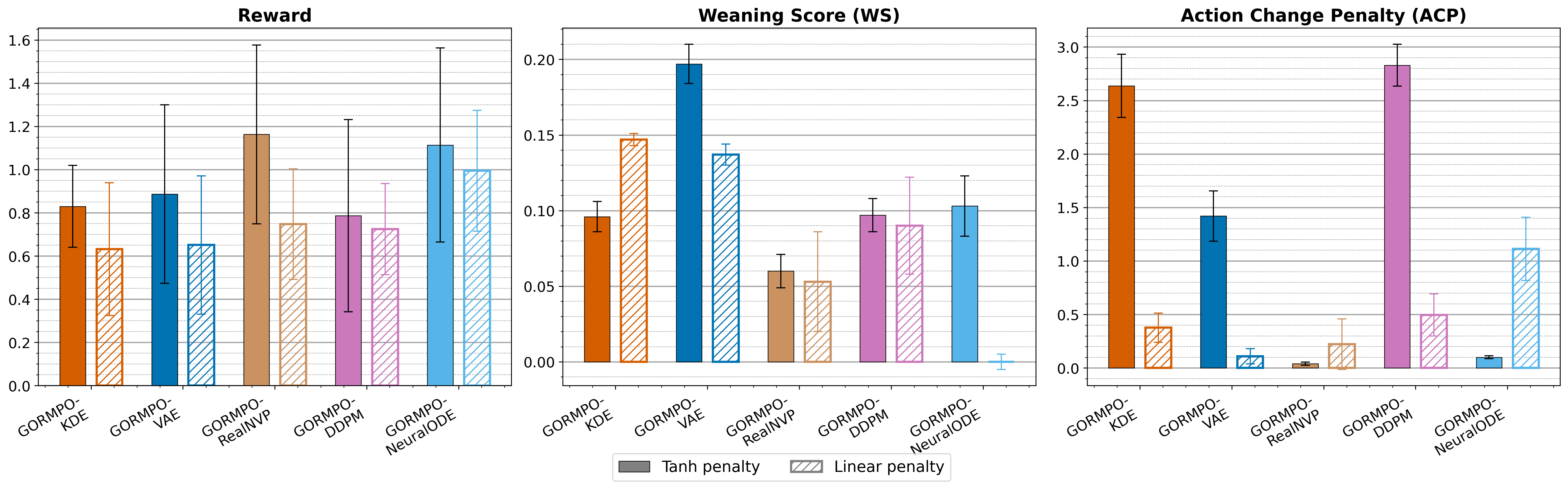}
    \caption{\textbf{Linear (non-saturating) penalty ablation on the medical dataset.} GORMPO with tanh penalty has large gains on reward and WS over the linear penalty models. }
    \label{fig:placeholder}
\end{figure*}

\clearpage
\section{Additional Results}

\label{sec:more_rl}

\paragraph{Different density formulations:} 
We model the joint density $p(s',a)$ to assess coverage of
next state-action pairs in the offline data.
This differs from the conditional transition density
$p(s'\mid s,a)$, which describes the distribution of next
states given a state-action pair but does not, by itself,
measure how frequently that pair occurs in the dataset.
Conditioning the density estimator on the current state through
$p(s',a\mid s)$ could provide more context-specific information;
however, we adopt the joint formulation to assess coverage
without this additional conditioning.
This choice also has a limitation of not
distinguishing previous states associated with the same
next state-action pair. We therefore treat joint density as
a coverage proxy, rather than a direct measure of
transition-model accuracy.

\begin{table}[h]
  \tiny
  \centering
  \caption{Negative Log-Likelihood (NLL) scores (lower the better) for all density estimators on all test datasets. While the NLL scores are not comparable on the same scale,  diffusion has the highest NLL because its likelihood is a strided ELBO bound rather than an exact likelihood.}

  \label{tab:nll_results}
  \begin{tabular}{lccccc}
  \toprule
  \textbf{Dataset} & \textbf{KDE} & \textbf{VAE} & \textbf{RealNVP} & \textbf{Diffusion} & \textbf{NeuralODE} \\
  \midrule
  halfcheetah-m-e-sparse
  & $2.948$
  & $5.199$
  & $-6.559$
  & $71.497$
  & $32.89$ \\

  hopper-m-e-sparse
  & $0.405$
  & $5.823$
  & $258.420$
  & $355134.44$
  & $-7.96$ \\

  walker2d-m-e-sparse
  & $0.939$
  & $2.863$
  & $-21.292$
  & $750839.06$
  & $-17.98$ \\

  \midrule

  medical dataset
  & $4.299$
  & $77.892$
  & $61.949$
  & $30220.5$
  & $80.398$ \\

  \bottomrule
  \end{tabular}
\end{table}

\begin{table}[ht]
    \tiny
    \setlength{\tabcolsep}{3pt}
\centering
\caption{Performance comparison on the medical dataset (mean $\pm$ std) over the same 10 seeds and 1000 evaluation episodes. Bold: best overall; underline: best GORMPO variant.}
\label{tab:abiomed_results}

\resizebox{\textwidth}{!}{
\begin{tabular}{l ccccc | ccccc}
\toprule
\multirow{2}{*}{\textbf{Metric}}
& \multicolumn{5}{c}{\textbf{Baselines}}
& \multicolumn{5}{c}{\textbf{GORMPO}} \\
\cmidrule(lr){2-6}
\cmidrule(lr){7-11}
& \texttt{SPOT} & \texttt{MBPO} & \texttt{MOBILE} & \texttt{GMB-TRPO} & \texttt{CORMPO}
& \texttt{KDE} & \texttt{VAE} & \texttt{RealNVP} & \texttt{DDPM} & \texttt{NeuralODE} \\
\midrule

Reward
& \result{0.462}{0.124}
& \result{0.596}{0.121}
& \result{0.589}{0.115}
& \bestresult{0.883}{0.119}
& \result{0.644}{0.127}
& \result{0.539}{0.120}
& \result{0.336}{0.123}
& \result{0.700}{0.115}
& \result{0.516}{0.129}
& \underline{\result{0.708}{0.124}} \\

WS
& \result{0.196}{0.008}
& \result{0.117}{0.005}
& \result{-0.434}{0.009}
& \result{0.064}{0.011}
& \result{0.064}{0.005}
& \result{0.213}{0.006}
& \result{0.220}{0.010}
& \underline{\bestresult{0.225}{0.009}}
& \result{0.145}{0.008}
& \result{0.037}{0.005} \\

ACP
& \result{0.039}{0.013}
& \result{0.589}{0.057}
& \result{3.485}{0.065}
& \result{0.620}{0.028}
& \bestresult{0.026}{0.008}
& \result{1.159}{0.060}
& \result{0.329}{0.032}
& \result{1.820}{0.081}
& \result{0.415}{0.037}
& \underline{\result{0.308}{0.042}} \\

\bottomrule
\end{tabular}
}
\end{table}

  \begin{table}[htbp]
\centering
\caption{\textbf{Performance comparison of MOBILE and density-regularized MOBILE variants} with VAE, RealNVP, and DDPM across medium datasets D4RL datasets averaged over 3 seeds. 
MOBILE-VAE and MOBILE-RealNVP improve on halfcheetah-medium and all variants on walker2d-medium. We import MOBILE's performance on medium datasets from its paper \cite{pmlr-v202-sun23q}. }
\label{tab:mobile_comparison}
\tiny
\setlength{\tabcolsep}{3pt}
\begin{tabular}{lc|ccc}
\toprule
\textbf{Datasets} & \texttt{MOBILE} & \texttt{MOBILE-VAE} & \texttt{MOBILE-RealNVP} & \texttt{MOBILE-DDPM} \\
\midrule

halfcheetah-m
&  $74.6 $ 
&  $\mathbf{76.5 \pm 0.686}$ 
&  $75.1\pm 0.99$ 
&  $73.36 \pm 1.79$ \\

hopper-m 
&  $\mathbf{107.0 }$ 
& $ 104.0 \pm 0.905 $ 
& $103.0\pm0.068 $ 
& $93.6  \pm    16.4 $ \\

walker2d-m
&  $87.7 $ 
&  $88.2\pm 2.08$ 
& $88.0 \pm 1.45$ 
& $ \mathbf{91.1   \pm    0.73}$ \\
\bottomrule
\end{tabular}
\vspace{1em}
\end{table}

\begin{figure*}[htb]
    \centering
    \includegraphics[width=1\linewidth]{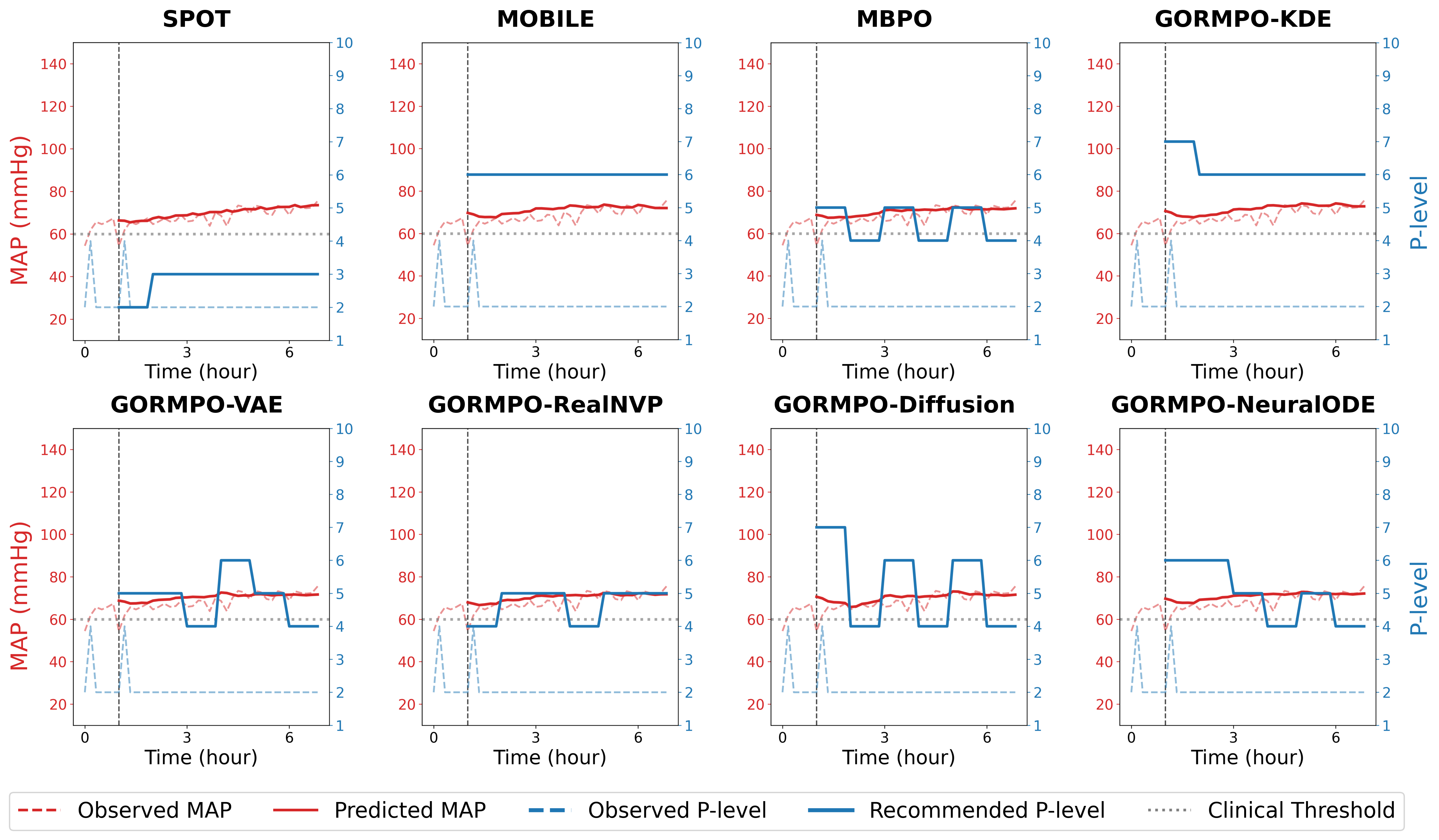}
    \caption{More visualizations of the policy performance on the medical dataset. We expect the policies to start weaning off the P-level (solid blue) as we observe a stable (above the clinical threshold and stationary) mean arterial pressure (MAP) signal. GORMPO-VAE and GORMPO-NeuralODE display the most successful weaning behavior with stable patient trajectory. GORMPO-RealNVP shows low magnitude in action change, resulting in 0 WS in this sample. Although MOBILE starts from an offset P-level of 6, it fails at weaning later.}
    \label{fig:abiomed_more_plots}
\end{figure*}

\clearpage
\section{GORMPO Hyperparameters and Model Architectures}
\label{sec:model_params}

\paragraph{Computing Resources.} Models are trained on NVIDIA A100 80GB and NVIDIA GeForce RTX 4090 GPUs. 

\begin{table}[!h]
\centering
    \caption{Elapsed time during GORMPO training on halfcheetah-medium-expert dataset.}
    \label{tab:placeholder}
\tiny
\begin{tabular}{cccccc}
\toprule
    \textbf{Time Elapsed (m) }   &  \texttt{GORMPO-KDE} &
\texttt{GORMPO-VAE} &
\texttt{GORMPO-RealNVP} &
\texttt{GORMPO-DDPM} &
\texttt{GORMPO-NeuralODE} \\
\midrule
      halfcheetah-m-e   &439 $\pm$  27.3& 478.0  $\pm$ 29.2 & 614.0  $\pm$169.9 & 379.6$\pm$16.4 & 573.8 $\pm$26.6\\
            \bottomrule
    \end{tabular}

\end{table}

\subsection{MBPO Parameters}
\label{sec:mbpo_params}
Following prior work highlighting that reliable policy selection often requires \emph{online} evaluation rather than offline policy evaluation \cite{DBLP:conf/nips/BrandfonbrenerW21, pmlr-v162-kurenkov22a}, we tune the penalty coefficient $\lambda$ using a small online interaction budget of 6 hyperparameters. Concretely, we sweep $\lambda$ with online rollouts for a single seed and evaluate for 100 episodes to limit interaction and reduce overfitting. This protocol mirrors common practice in offline RL, where hyperparameters are chosen under an explicit online evaluation budget \cite{pmlr-v162-kurenkov22a} while maintaining a conservative use of environment interaction. See Table \ref{tab:lambda} for the best $\lambda$.

An alternative to online tuning is off-policy evaluation (OPE). For the medical task, the TDT simulator~\cite{tumay2025guardian} already serves as a model-based OPE method. We avoid importance sampling (IS) and doubly robust (DR) estimators because both rely on an estimated behavior policy, whose errors compound with high-variance importance weights in exactly the low-coverage regions our method targets. Moreover, DR reduces variance but inherits the same weights. Fitted Q evaluation (FQE) faces a similar issue, as it extrapolates Q-values into regions the data rarely covers. We therefore use a small online tuning budget and show in Figure \ref{fig:lambda_sensitivity} that GORMPO is largely insensitive to $\lambda$.

\begin{table}[!h]
    \centering
    \caption{Finetuned $\lambda$ values for each sparse dataset. m=medium, r=replay, e=expert.}
    \label{tab:lambda}
    \small
    \setlength{\tabcolsep}{2pt}
    \begin{tabular}{cccccc}
    \toprule
        \textbf{Models} &\texttt{GORMPO-KDE} &
\texttt{GORMPO-VAE} &
\texttt{GORMPO-RealNVP} &
\texttt{GORMPO-DDPM} &
\texttt{GORMPO-NeuralODE} \\
\midrule
        medical dataset  &0.2&0.1&0.2& 0.4& 0.2\\
        \midrule
         halfcheetah-m-e-sparse &0.1 &0.1&0.8& 0.3& 0.1\\
          hopper-m-e-sparse& 0.05&0.3&0.8&0.05& 0.5\\
         walker2d-m-e-sparse&0.5 &0.5&0.8&0.05&0.05\\



         \bottomrule
    \end{tabular}
\end{table}

We directly use the same hyperparameters for MBPO training in GORMPO as reported in MOPO \cite{yu2020mopo} except for the rollout length of the transition model. We tuned this parameter to achieve stable Q-values during training, which is essential for strong policy optimization, shown in Table \ref{tab:rollout_length}. See Table \ref{tab:mbpo_param} for base parameters of MBPO for all datasets. GORMPO is trained up to 1,000,000 steps on sparse D4RL datasets, and 200,000 steps on the medical dataset. We also give the same tuning budget of 6 hyperparameters to SPOT \cite{wu2022spot} and MOBILE's penalty parameters \cite{pmlr-v202-sun23q}.

\begin{table}[!h]
    \centering
    \caption{Rollout length $h$ of MBPO-based GORMPO for each sparse dataset. m=medium, e=expert.}
    \label{tab:rollout_length}
    \small
    \setlength{\tabcolsep}{2pt}
    \begin{tabular}{cccccc}
    \toprule
        \textbf{Models} &\texttt{GORMPO-KDE} &\texttt{GORMPO-VAE} &\texttt{GORMPO-RealNVP} &\texttt{GORMPO-DDPM} &\texttt{GORMPO-NeuralODE} \\\midrule
        medical dataset        &5 &5 &5 &5 &5\\
        \midrule
        halfcheetah-m-e-sparse &5 &5 &5 &5 &5\\
        hopper-m-e-sparse      &5 &3 &5 &5 &3\\
        walker2d-m-e-sparse    &5 &5 &5 &5 &3\\
    \bottomrule
    \end{tabular}
\end{table}

\begin{table*}[!t]
\centering
\caption{Base hyperparameters of our GORMPO implementation.}
\label{tab:mbpo_param}
\begin{tabular}{p{0.45\linewidth} c}
\toprule
\textbf{Parameters} & \textbf{Value} \\
\midrule
Actor learning rate & $3\times10^{-4}$ \\
Critic learning rate & $3\times10^{-4}$ \\
Discount factor ($\gamma$) & $0.99$ \\
Target network update coefficient  & $0.005$ \\
Target entropy  & $-1$ \\
Temperature optimizer learning rate & $3\times10^{-4}$ \\
Dynamics model learning rate & $1 \times10^{-3}$ \\
Dynamics ensemble size & $7$ \\
Holdout ratio & $0.2$ \\
Steps per epoch & $1000$ \\
Evaluation episodes & $1000$ \\
Mini-batch size & $256$ \\
Model rollout horizon & $5$ \\
Rollout batch size & $10000$ \\
Rollout frequency & $1000$ \\
Real-to-model data sampling ratio & $0.05$ \\
\bottomrule
\end{tabular}
\end{table*}

\begin{table}[!h]
    \centering
    \caption{Finetuned $\lambda$ values of MOBILE-based GORMPO for each sparse dataset. m=medium, e=expert.}
    \label{tab:lambda_mobile}
    \small
    \setlength{\tabcolsep}{2pt}
    \begin{tabular}{cccccc}
    \toprule
        \textbf{Models} &\texttt{GORMPO-KDE} &\texttt{GORMPO-VAE} &\texttt{GORMPO-RealNVP} &\texttt{GORMPO-DDPM} &\texttt{GORMPO-NeuralODE} \\\midrule
        halfcheetah-m-e-sparse &0.3 &0.3 &0.8 &0.5 &0.3\\
        hopper-m-e-sparse      &0.3 &0.3 &0.8 &0.5 &0.3\\
        walker2d-m-e-sparse    &0.3 &0.7 &0.5 &0.5 &0.3\\
    \bottomrule
    \end{tabular}
\end{table}
For MOBILE-based GORMPO experiments, we use the hyperparameters reported in Table \ref{tab:lambda_mobile} for each sparse D4RL dataset. Also, the MOBILE runs used a MOBILE penalty coefficient $\beta$ of 1.0 for halfcheetah and 1.5 for hopper and walker2d, and a rollout length of 5, 5, and 1, respectively.

\subsection{Density Estimator Parameters}
\label{sec:density_params}

In this section, we include hyperparameters and training settings of different density estimator models.
All density estimators use the 1st percentile of the estimated validation distribution as the model threshold. The input dimension of each model for the medical, halfcheetah, walker2d, and hopper tasks is 73, 23, 23, 14, respectively. 
\paragraph{Kernel Density Estimator.} We employ the FAISS \cite{Johnson2019FaissTBD} k-nearest neighbor model with GPU acceleration. See Table \ref{tab:kde} for parameters.
 \begin{table}[!h]                                                                                   
  \centering     
  \caption{KDE Model and Training Parameters}  
  \label{tab:kde} 
  
  \begin{tabular}{ll}                                                                                
  \toprule                                                                                           
  \textbf{Parameter} & \textbf{Value} \\                                                       
  \midrule                                                                                             
  Kernel Function & Gaussian \\                                                                       
  Bandwidth & 1.0 \\                                                                                  
  Number of Neighbors (k) & 100 \\                        
  \bottomrule                                                                                          
  \end{tabular}                                                                                         
  \end{table}
  
\paragraph{Variational Autoencoder.} See Table \ref{tab:vae}.

 \begin{table}[!h]                                                                                 
  \centering    
  \caption{VAE Model and Training Parameters}  
     \label{tab:vae}   
  \begin{tabular}{ll}                                                                                   
  \toprule                                                                                            
  \textbf{Parameter} & \textbf{Value} \\                                                     
  \midrule                                                                                            
  Latent Dimension & 16 \\
  Hidden Dimensions & [256, 256] \\                                                                   
  Activation Function & ReLU \\         
  \midrule                               
  Training Epochs & 100 \\                                                                            
  Batch Size & 256 \\                                                                                 
  Optimizer & Adam \\                                                                                
  Learning Rate & $1 \times 10^{-3}$ \\                                                               
  Weight Decay & 0 \\                                                                                 
  Beta ($\beta$-VAE weight) & 1.0 \\                                                                   
  LR Scheduler & ReduceLROnPlateau \\                                                                 
  Scheduler Factor & 0.5 \\                                                                          
  Scheduler Patience & 7 \\                                                               
  Early Stopping Patience & 15 \\                                  
  \bottomrule                                                                                         
  \end{tabular}                                                                                        
  \end{table}   
  
\paragraph{RealNVP.} See Table \ref{tab:realnvp}.
                            
  \begin{table}[!h]                                                                                
  \centering   
  \caption{RealNVP Training Parameters}    
  \label{tab:realnvp}  
  \begin{tabular}{ll}                                                                                
  \toprule                                                                                          
  \textbf{Parameter} & \textbf{Value} \\                                                    
  \midrule      
  Number of Coupling Layers & 6 \\                                                                   
  Hidden Dimensions & [256, 256] \\                                                                  
  Activation Function & ReLU \\                                                                      
  Scale Stabilization & tanh \\  
  Prior Distribution & Standard Normal ($\mu=0$, $\sigma=1$) \\     
  Training Epochs & 100 \\                                                                           
  Batch Size & 256 \\                                                                                
  Optimizer & Adam \\                                                                                
  Learning Rate & $1 \times 10^{-3}$ \\                                                              
  Weight Decay & 0 \\                                                                                
  LR Scheduler & ReduceLROnPlateau \\                                                                
  Scheduler Factor & 0.5 \\                                                                          
  Scheduler Patience & 7 \\                                                              
  Early Stopping Patience & 15 \\                                                                                                             
  \bottomrule                                                                                        
  \end{tabular}           
  \end{table} 
  
\paragraph{Diffusion.} See Table \ref{table:diffusion-hyperparameters}.

\begin{table}[!h]
\centering
\caption{Diffusion Architecture and Training Hyperparameters.}
\label{table:diffusion-hyperparameters}
\begin{tabular}{lc}
\toprule
\textbf{Parameter} & \textbf{Value} \\
\midrule
Parameters & $513d + 590,848 ~(610K)$ \\
Observation Dimension ($d$) & $\{14, 23,73\}$ \\
Training Epochs & 50 \\
Batch Size & 512 \\
\midrule
Model Type & DDPM / DDIM \\
Noise Predictor & Time-conditioned MLP \\
Time Embedding & Sinusoidal (128-dim) \\
Hidden Layers & 3 \\
Hidden Units & 512 \\
Activation & SiLU \\
Input Structure & $[\mathbf{x}_t, \text{emb}(t)] \in \mathbb{R}^{d+128}$ \\
Output Structure & $\boldsymbol{\epsilon} \in \mathbb{R}^{d}$ \\
\midrule
Optimizer & AdamW \\
Learning Rate & $2 \times 10^{-4}$ \\
Beta Schedule & Linear \\
\bottomrule
\end{tabular}
\vspace{0.5em}
\end{table}

\paragraph{NeuralODE.} See Table \ref{table:neuralode-hyperparameters}.

\begin{table}[!h]
\centering
\caption{NeuralODE Architecture and Training Hyperparameters.}
\label{table:neuralode-hyperparameters}
\begin{tabular}{lc}
\toprule
\textbf{Parameter} & \textbf{Value} \\
\midrule
Parameters & $1025d + 263,680~(280K)$ \\
Observation Dimension ($d$) & $\{14, 23,73\}$ \\
Training Epochs & 20 \\
Batch Size & 512 \\
\midrule
Network Type & Time-conditioned MLP \\
Hidden Layers & 2 \\
Hidden Units & 512 \\
Activation & SiLU \\
Input Structure & $[\mathbf{x}, t] \in \mathbb{R}^{d+1}$ \\
Output Structure & $\mathbf{v} \in \mathbb{R}^{d}$ \\
\midrule
Optimizer & AdamW \\
Learning Rate & $1 \times 10^{-3}$ \\
Weight Decay & $1 \times 10^{-4}$ \\
\bottomrule
\end{tabular}
\vspace{0.5em}
\end{table}

\clearpage
\section{OOD Dataset Generation Details}
\label{sec:ood_data}

We create ``OOD datasets'' by randomly selecting 5 trajectories, adding Gaussian noise ($\mathcal{N}(\mu, 0.1)$) to observations and actions, and concatenating the noiseless subset of the dataset with the noisy version. Different levels of OOD datasets correspond to varying values of $\mu$ of the Gaussian noise. We depict the state norm versus action norm plots in Figures \ref{fig:ood_datasets_abiomed} and \ref{fig:ood_datasets_d4rl}. The number of samples in each dataset is 400, 7992, 10000, 9546, for the medical, hopper, halfcheetah, and walker2d datasets, proportional to the original dataset size.

\begin{figure*}[!h]
    \centering
    \includegraphics[width=1\linewidth]{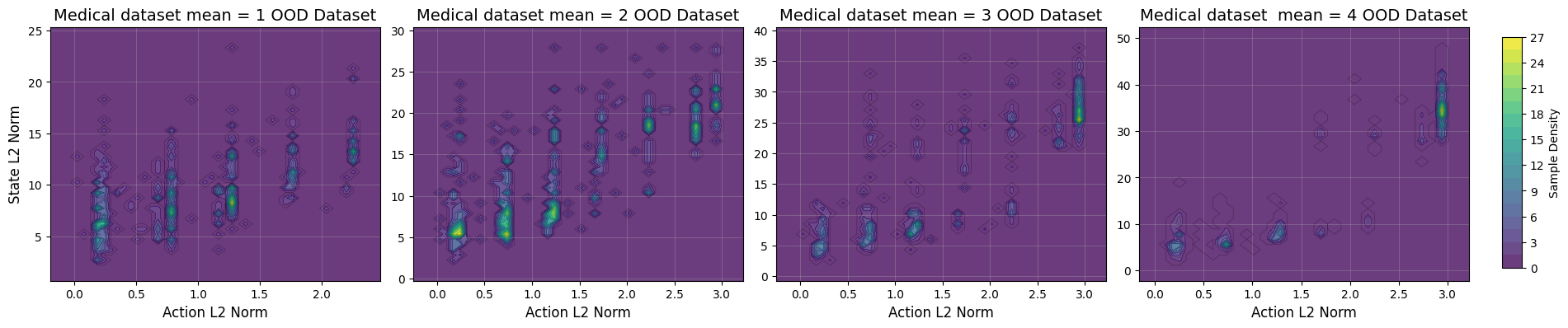}
    \caption{OOD dataset visualizations for the medical dataset.}
    \label{fig:ood_datasets_abiomed}
\end{figure*}
 
\begin{figure*}[!h]
    \centering
    \begin{subfigure}[b]{\linewidth}
        \centering
        \includegraphics[width=\linewidth]{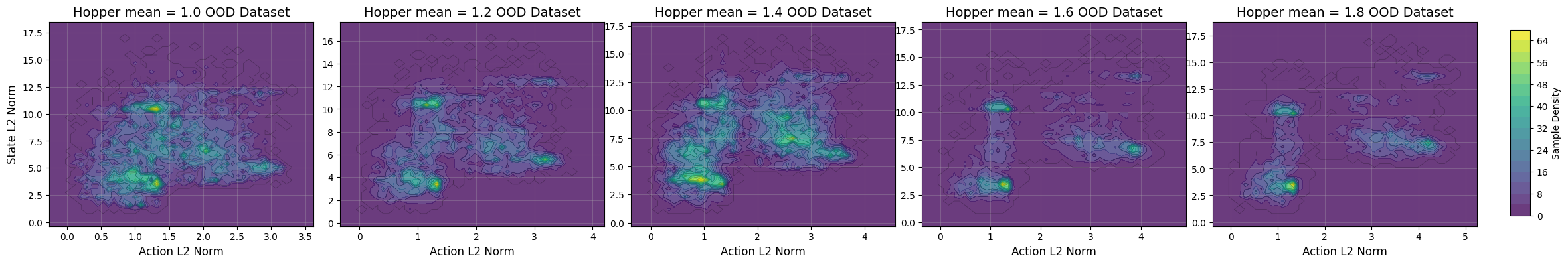}
        \caption{hopper-medium-expert OOD dataset visualization.}
        \label{fig:hopper_ood}
    \end{subfigure}
    \hfill
    \begin{subfigure}[b]{\linewidth}
        \centering
        \includegraphics[width=\linewidth]{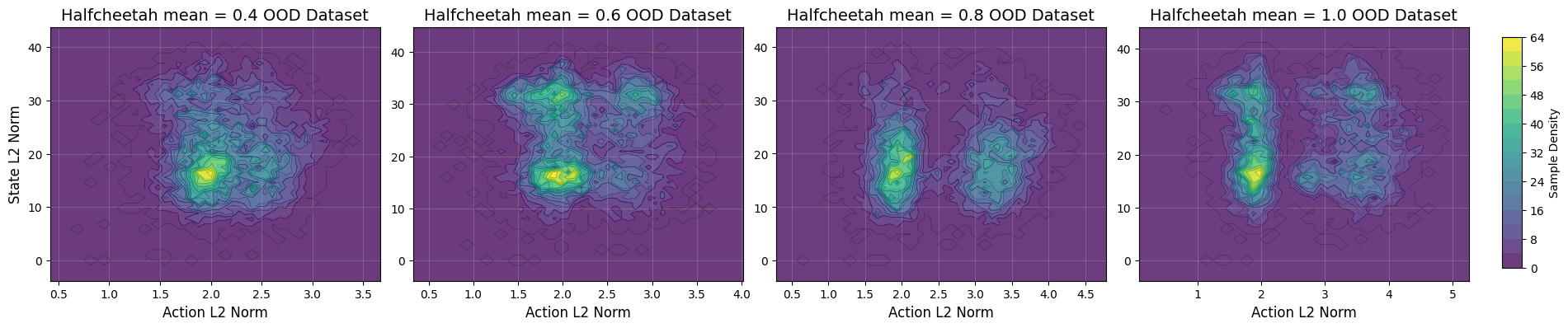}
        \caption{halfcheetah-medium-expert OOD dataset visualization.}
        \label{fig:halfcheetah_ood}
    \end{subfigure}
    \hfill
    \begin{subfigure}[b]{\linewidth}
        \centering
        \includegraphics[width=\linewidth]{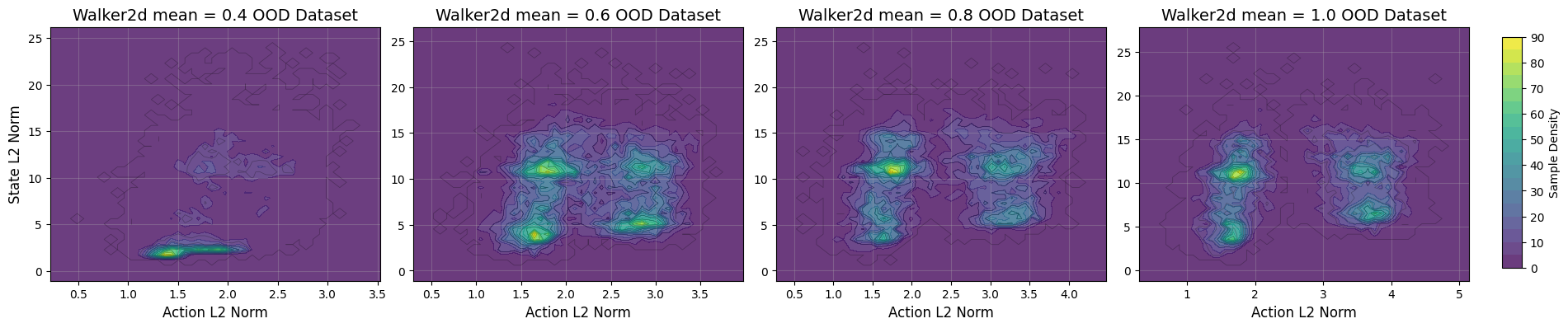}
        \caption{walker2d-medium-expert OOD dataset visualization.}
        \label{fig:walker2d_ood}
    \end{subfigure}
    \caption{OOD dataset visualizations for sparse D4RL medium-expert datasets.}
    \label{fig:ood_datasets_d4rl}
\end{figure*}


\section{More Details on the Smart Weaning of Mechanical Circulatory Devices Task and Dataset}
\label{sec:abiomed}
We use the same medical task and real-world dataset in CORMPO \cite{tumay2025guardian}.

\subsection{Dataset Details} 
Our real-world medical dataset comprises 379 patients corresponding to 17865 samples after downsampling and sliding window processes. 
This dataset includes 12 features recorded directly or derived from signals of the MCS device, namely: Mean arterial pressure (MAP), mean pump speed, mean motor current, mean pump flow, left ventricular pressure (LVP), left ventricular end diastolic pressure (LVEDP), heart rate (HR), systolic blood pressure (SBP), diastolic blood pressure (DBP), pulsatility, relaxation constant (Tau\_LV), and elastance estimation (ESE\_LV). We downsample the original signal of 25 Hz into 0.00167 Hz (1 sample per 10 minutes) and extract samples with a sliding window of 1 hour  \cite{tumay2025guardian}. 

\subsection{MDP Design for MCS} We first formulate the MCS weaning problem as an MDP. The challenge in formulating the environment is balancing the rich information of medical time series with the learning challenges of a high-dimensional Markov Decision Process (MDP). 
We define each \textit{state} in the MDP to consist of $t$ time-steps of $k$ different physiological features, 
i.e., $\mathcal{S} \subseteq \mathbb{R}^{t \times k}$. The \textit{action} space is $\mathcal{A} = \{2,3,\cdots, 9\}$, corresponding to pump level P2 to P9 on the MCS device. The objective is to optimize patient outcome with a clinically appropriate weaning strategy. For the offline RL problem, we organize the patient data into a replay buffer dataset of $\mathcal{D} = \{(s_i,a_i, s'_i, r_i)\}_i$  according to the formulation. The state space, action space, reward, and MDP design is informed by expert recommendation \cite{tumay2025guardian}.

\paragraph{Observations. }  The observation space includes 12 hemodynamic features of the patient. Our inputs are the pump pressure, pump speed, and motor current 25 Hz signals recorded by the MCS device. 
 We down-sample patient data from 25Hz to 0.00167Hz (1 sample per 10 minutes) and process them into sliding windows of 1 hour (6 time steps) to be used as states for digital twin prediction and decision making based on expert suggestion.
    Therefore, the observation space is $\mathcal{S} = \mathbb{R}^{6 \times 12}$, where each $s_i = x_{t: t+6}$ at some $t$ for a patient. 
\paragraph{Action. } The action for our MDP is the pump support level (P-level) of the MCS device. The device operates at 8 different speed levels, from P2-P9, each with a constant motor speed (rpm). The P-level proportionally determines the blood flow provided to the patient by the motor’s speed and current. Clinicians can control the P-level while the patient is on support.
The P-level generally stays unchanged in 1-hour intervals, unlike the state features, since it is manually controlled by the clinicians during the treatment. In practice, we take the mean P-level over the 1-hour interval as expert action. As a result, we define $ \mathcal{A} = \{2, \dots, 9\}$.

\paragraph{Rewards. } The design table for the reward function in Appendix \ref{sec:A} is generated in line with medical consultancy. It assigns a (inverted) risk score based on acceptable intervals for hemodynamic features. 
The physiological reward is further normalized through Z-score normalization and clipped between $[-2,2]$ to ensure training stability.

\subsection{Medically-informed Metrics}
\label{sec:A}
\textbf{Action Change Penalty (ACP)} \cite{ogsrl}: Abrupt and extreme changes in P-level may maximize rewards; however, they can induce physiological instability in a real-world setting. ACP gauges policy volatility and is given by:
\begin{equation*}
    \text{ACP =} \sum^{T}_{i=1}||a_{i-1}-a_{i}||_2, \text{ if }  ||a_{i-1}-a_{i}||_2 >2.
\end{equation*}

where $a_{i-1}$ is an action at state $i-1$, $a_{i}$ is a subsequent action, and $T$ is the episode length. Lower ACP values indicate stable physiology and safe weaning, but note that a value of 0 is undesirable as the P-level must be lowered for weaning. 

\textbf{Weaning Score (WS)}:
To capture satisfactory weaning patterns, we support P-level reductions at most every 1 hour when the patient is observed as hemodynamically stable for the past 1 hour as depicted in Eq. \ref{eq:WS}. Higher weaning scores signify an appropriate reduction in P-level during relatively stable physiological states. We employ stability based on the gradient of the past hemodynamic state as 


\begin{equation}
\begin{aligned}
\texttt{Is\_Stable}(i) = \left|\frac{\partial\,\mathrm{MAP}(i)}{\partial t}\right| < \kappa_{\mathrm{MAP}} \wedge \left|\frac{\partial\,\mathrm{HR}(i)}{\partial t}\right| < \kappa_{\mathrm{HR}} \wedge \left|\frac{\partial\,\mathrm{Pulsat}(i)}{\partial t}\right| < \kappa_{\mathrm{Pulsat}},
\end{aligned}
\label{eq:second_ws}
\end{equation}
where $\kappa_{\mathrm{MAP}} = 1.36$, $\kappa_{\mathrm{HR}} = 2.16$, and $\kappa_{\mathrm{Pulsat}} = 1.95$, indicating a proxy for stability with a low gradient value of 3 hemodynamic indicators in the past state chosen with statistical significance tests.

 We evaluate the policies with the gradient-based WS definition. Second part of the WS metric is as follows.
\begin{equation*}
\texttt{Weaned}(i) =
\begin{cases} 
-1, & \text{if } a_{i} - a_{i+1} < 0, \\
a_{i} - a_{i+1}, & \text{if } a_{i} - a_{i+1} \in \{1,2\}, \\
0, &  \text {otherwise}. \\
\end{cases}
\end{equation*}

\textbf{Physiological Reward}: The reward generally reflects the well-being of the patient, according to the mean arterial pressure (MAP), heart rate (HR), and pulsatility of the past hour. Our design follows the clinically defined ranges for hemodynamic stability while caring for the smoothness and differentiability of the function.

\begin{table}[ht]
    \centering
    \caption{Hemodynamic instability score table from \cite{buitenwerf2019haemodynamic}. We use a modified version of this table as our physiological reward. When used for evaluating the learned policy as a reward function, we multiply the risk score by -1. }
    \label{tab:reward_table}
    \begin{tabular}{clll}
    \toprule
    \textbf{Score Component} & \textbf{Value} & \textbf{Score} \\ \midrule
    Hemodynamic \\
    Variable & MAP $\geq60$ & 0 \\
       & 50 to 59 & 1 \\
      & 40 to 49 & 3 \\
      & $<40$ & 7 \\
    \midrule
      Minimum MAP  & $\geq60$ & 0 \\
      in window & 50 to 59 & 1 \\
      & 40 to 49 & 3 \\
      & $<40$ & 7 \\
     \midrule
      Time Spent MAP& 0 & 0 \\
       $<60$ mmHg (\%)  & 2 & 1 \\
     & 5 & 3 \\
       & $>5$ & 7 \\
     \midrule
     Pulsatility & $>20$ & 0 \\
       & 10-20 & 5 \\
       & $<10$ & 7 \\
     \midrule
      HR & $>100$ & 3 \\
       & $<50$ & 3 \\
     \midrule
      LVEDP & $>20$ & 7 \\
       & 15 to 20 & 4 \\
       & $<15$ & 3 \\ 
    \midrule
      CPO & 0.6 to 1 & 1 \\ 
       & $<0.6$ & 3 \\
       & $<0.5$ & 5 \\
     \bottomrule
    \end{tabular}
    \vspace{1em}
    \end{table}

The reward design in Table \ref{tab:reward_table} is staircase-shaped, which has two drawbacks:  non-differentiability and a sparse signal. We reformulate the hemodynamic instability score in the following way.

\begin{itemize}
    \item  \textbf{Heart Rate Penalty Function} The heart rate penalty function penalizes deviations from an optimal heart rate of 75 bpm using a quadratic penalty:
\begin{equation}
P_{\text{hr}}(hr) = \text{ReLU}\left(\frac{(hr - 75)^2}{250} - 1\right)
\end{equation}
where $\text{ReLU}(x) = \max(0, x)$. This function has zero penalty for heart rates in the range $[59.2, 90.8]$ bpm and applies quadratic penalties for heart rates outside this range.

\item \textbf{Minimum MAP Penalty Function} The minimum Mean Arterial Pressure (MAP) penalty function ensures MAP values remain above 60 mmHg:
\begin{equation}
P_{\text{minMAP}}(MAP) = \text{ReLU}\left(\frac{7(60 - MAP)}{20}\right)
\end{equation}
This function applies a linear penalty when MAP falls below 60 mmHg, with the penalty increasing as MAP decreases further from this threshold.

\item  \textbf{Pulsatility Penalty Function} The pulsatility penalty function maintains pulsatility within the range $[20, 50]$:

\begin{equation}
P_{\text{pulsat}}(p) = \text{ReLU}\left(\frac{7(20 - p)}{20}\right) + \text{ReLU}\left(\frac{p - 50}{20}\right)
\end{equation}

This bi-directional penalty function penalizes pulsatility values below 20 and above 50, with zero penalty for pulsatility in the range $[20, 50]$.

\item \textbf{Hypertension Penalty Function} The hypertension penalty function penalizes elevated mean MAP values above 106 mmHg:

\begin{equation}
P_{\text{hyp}}(MAP) = \text{ReLU}\left(\frac{MAP - 106}{18}\right)
\end{equation}

This function applies a linear penalty for mean MAP values exceeding the hypertension threshold of 106 mmHg.
\end{itemize}

The overall reward function combines all penalty components and negates the sum to create a reward signal:

\begin{equation}
R(s) = -\left[P_{\text{minMAP}}(\min(\text{MAP})) + P_{\text{hyp}}(\overline{\text{MAP}}) + P_{\text{hr}}(\min(HR)) + P_{\text{pulsat}}(\min(\text{Pulsat}))\right]
\end{equation}

where:
\begin{itemize}
    \item $\min(\text{MAP})$, $\min(HR)$, $\min(\text{Pulsat})$ are the minimum values over the time horizon
    \item $\overline{\text{MAP}}$ is the mean MAP over the time horizon
    \item The negative sign converts penalties into rewards (higher rewards for lower penalties)
\end{itemize}

\section{Details and Visualization of Sparse D4RL Datasets}
\label{sec:more_data}
 We show the details of sparse dataset generation from the Gym-MuJoCo D4RL medium-expert datasets. The unsafe region limits are decided based on the visualizations in Figure \ref{fig:sparsity_grid} on Action Norm vs Reward contour plots. We display the selected limits in Table \ref{tab:unsafe_region}. In general, we enclose one mode of the action norm vs reward space.

 \begin{table*}[!h]
     \centering
          \caption{Sparse D4RL action norm and reward unsafe region limits. ha=halfcheetah, ho=hopper, wa=walker2d.}
     \label{tab:unsafe_region}
     \small
     \setlength{\tabcolsep}{2pt}
     \begin{tabular}{l|cccccc}
     \toprule
          & \makecell{Minimum\\reward}
    & \makecell{Maximum\\Reward}
    & \makecell{Minimum\\Action Norm}
    & \makecell{Maximum\\Action Norm}
    & \makecell{Discarded Percentage\\from the unsafe region}
    & \makecell{Discarded Percentage\\from the full dataset}\\
          \midrule
         ha-m-e-sparse & $-3.014$ &$6.569$& $0.363$&$1.936$ & $41\%$ & $27.5\%$\\
          ho-m-e-sparse& $0.549$ & $ 3.892$&$0.012$&$1.058$&$40\%$& $22\%$\\
          wa-m-e-sparse&$-2.557$&$4.579$& $0.255$&$1.762$&$48\%$& $27\%$\\
          \bottomrule
     \end{tabular}
 \end{table*}

\begin{figure*}[!h]
    \centering

    \begin{subfigure}[b]{0.32\linewidth}
        \centering
        \includegraphics[width=\linewidth]{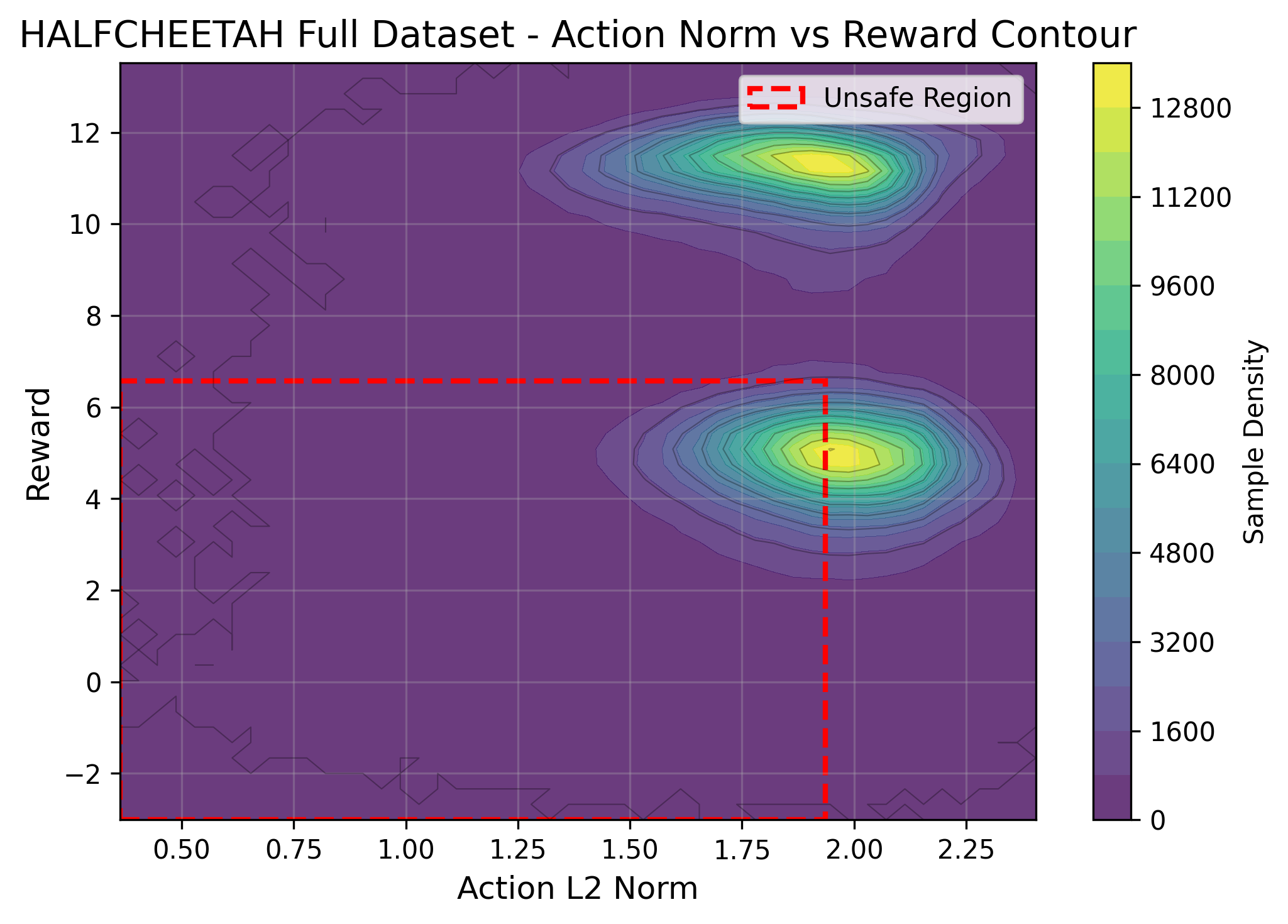}
        \caption{halfcheetah (Full)}
        \label{fig:full_halfcheetah}
    \end{subfigure}
    \hfill
    \begin{subfigure}[b]{0.32\linewidth}
        \centering
        \includegraphics[width=\linewidth]{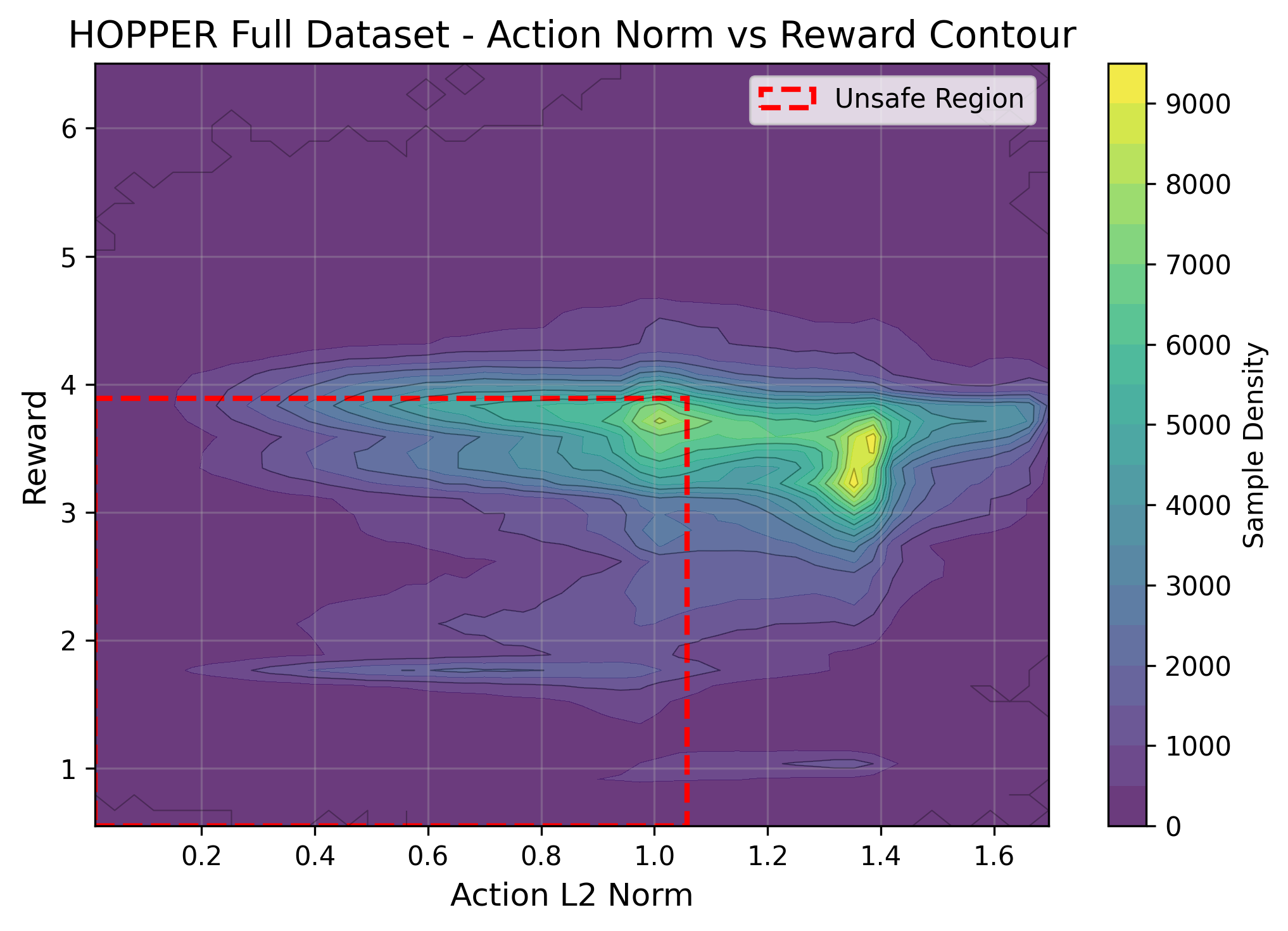}
        \caption{hopper (Full)}
        \label{fig:full_hopper}
    \end{subfigure}
    \hfill
    \begin{subfigure}[b]{0.32\linewidth}
        \centering
        \includegraphics[width=\linewidth]{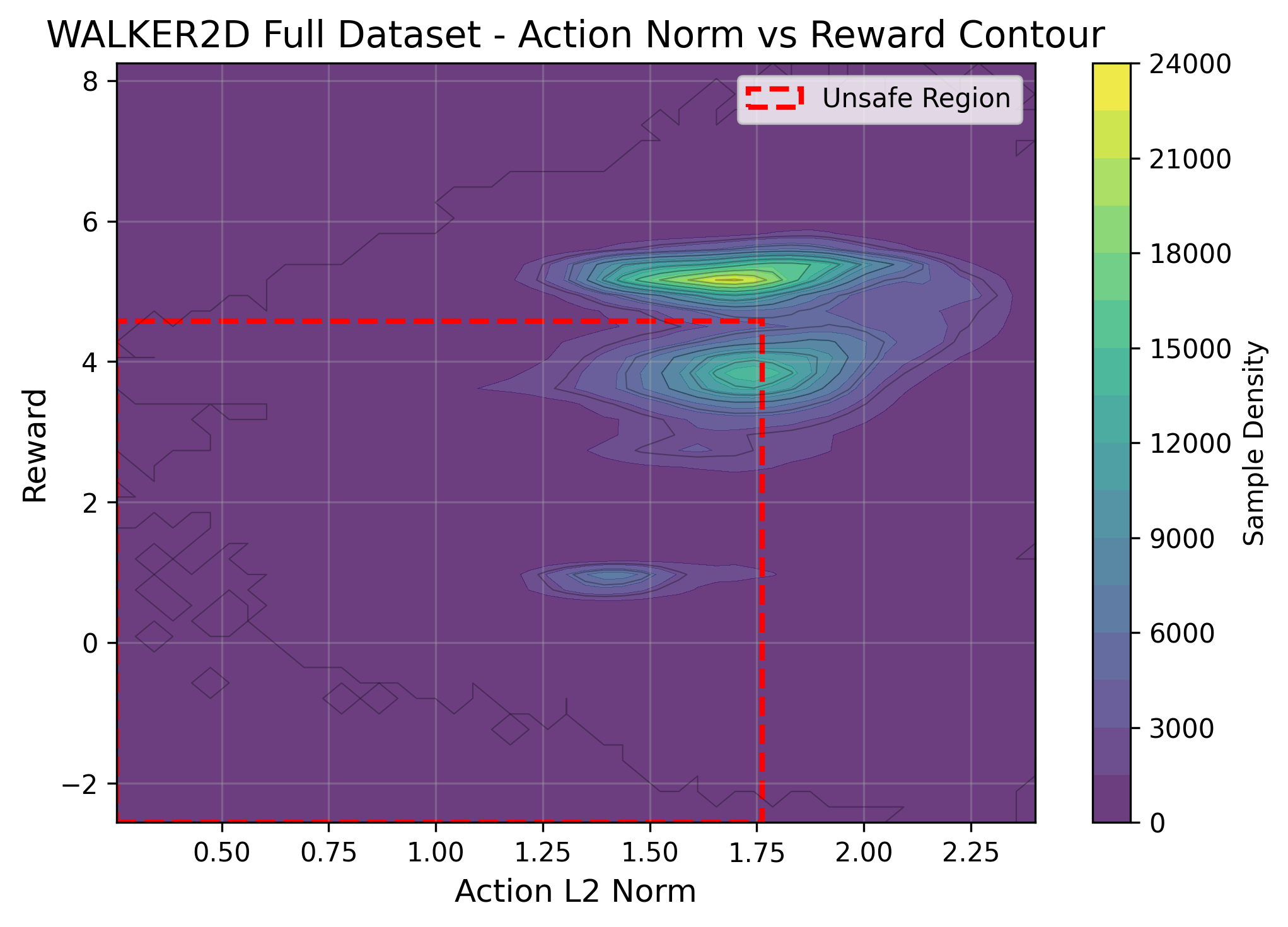}
        \caption{walker2d (Full)}
        \label{fig:full_walker2d}
    \end{subfigure}

    \vspace{0.6em}

    \begin{subfigure}[b]{0.32\linewidth}
        \centering
        \includegraphics[width=\linewidth]{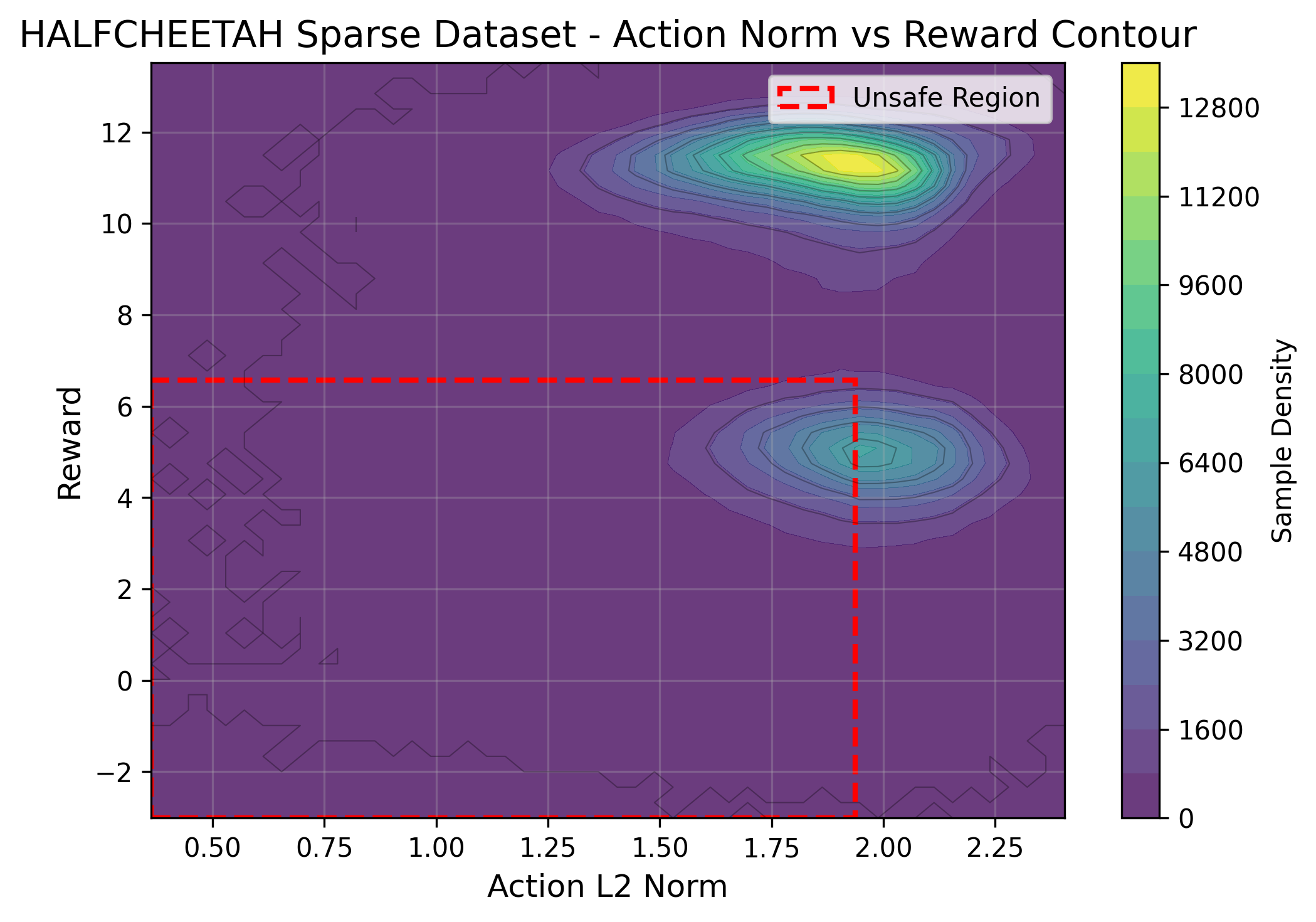}
        \caption{halfcheetah (Sparse)}
        \label{fig:sparse_halfcheetah}
    \end{subfigure}
    \hfill
    \begin{subfigure}[b]{0.32\linewidth}
        \centering
        \includegraphics[width=\linewidth]{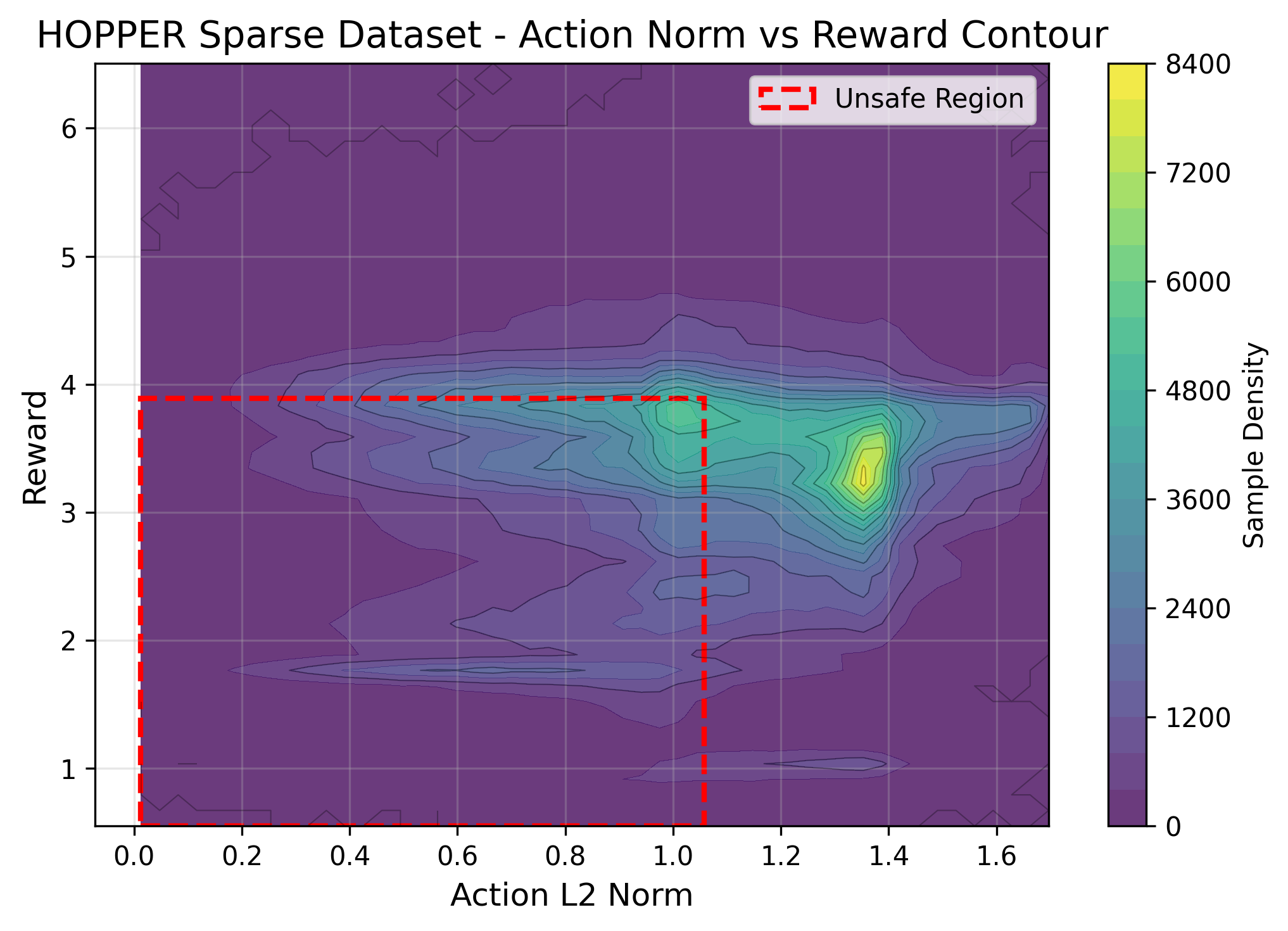}
        \caption{hopper (Sparse)}
        \label{fig:sparse_hopper}
    \end{subfigure}
    \hfill
    \begin{subfigure}[b]{0.32\linewidth}
        \centering
        \includegraphics[width=\linewidth]{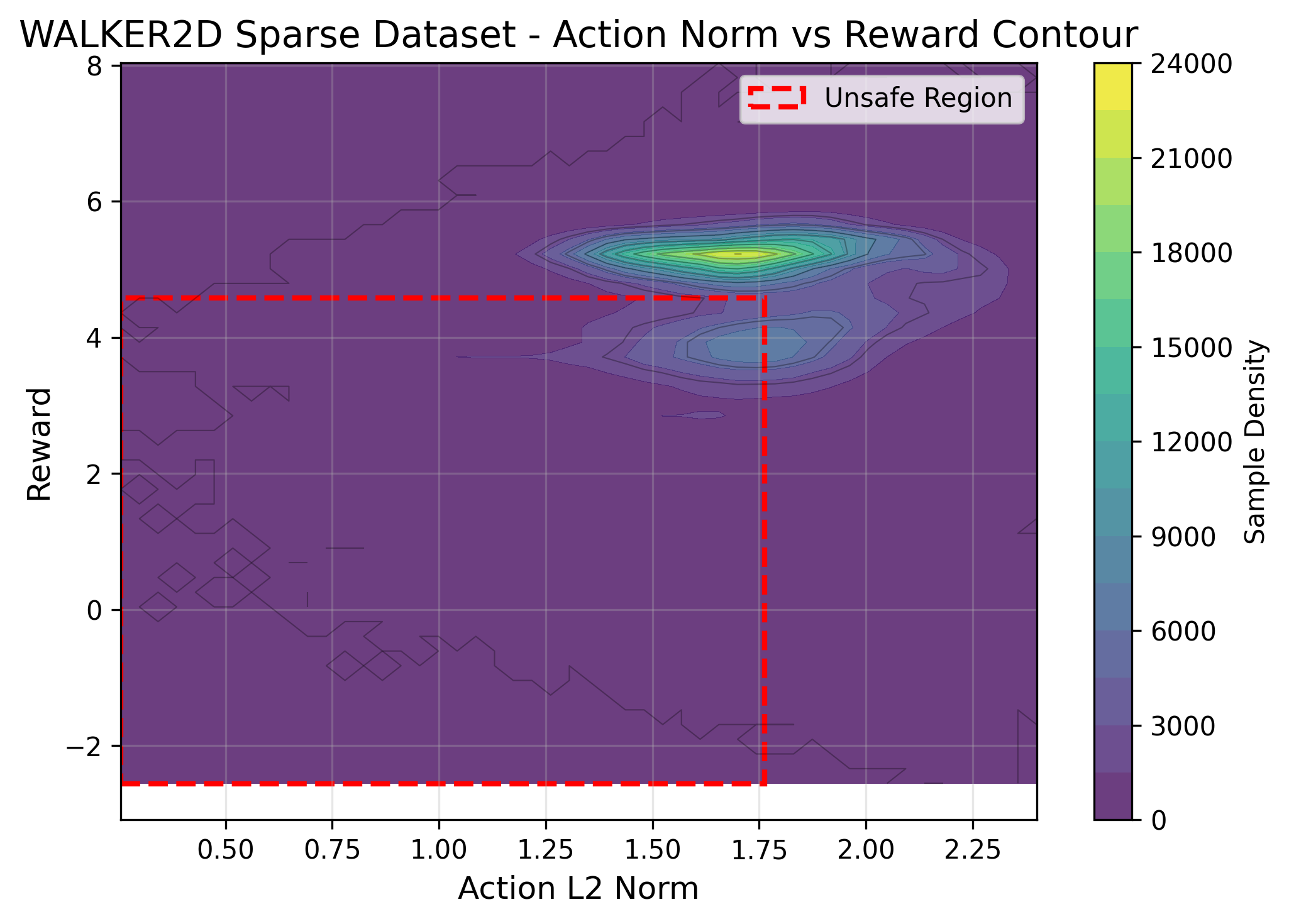}
        \caption{walker2d (Sparse)}
        \label{fig:sparse_walker2d}
    \end{subfigure}

    \caption{State-action space visualizations for D4RL datasets. Top row: full datasets. Bottom row: sparse datasets. Overall, we discard 27.5\%, 22\%, and 27\% of the samples from halfcheetah, hopper, and walker2d, respectively.}
    \label{fig:sparsity_grid}
\end{figure*}

\begin{figure*}[!h]
\centering
    \begin{subfigure}[b]{0.25\linewidth}
        \centering
        \includegraphics[width=\linewidth]{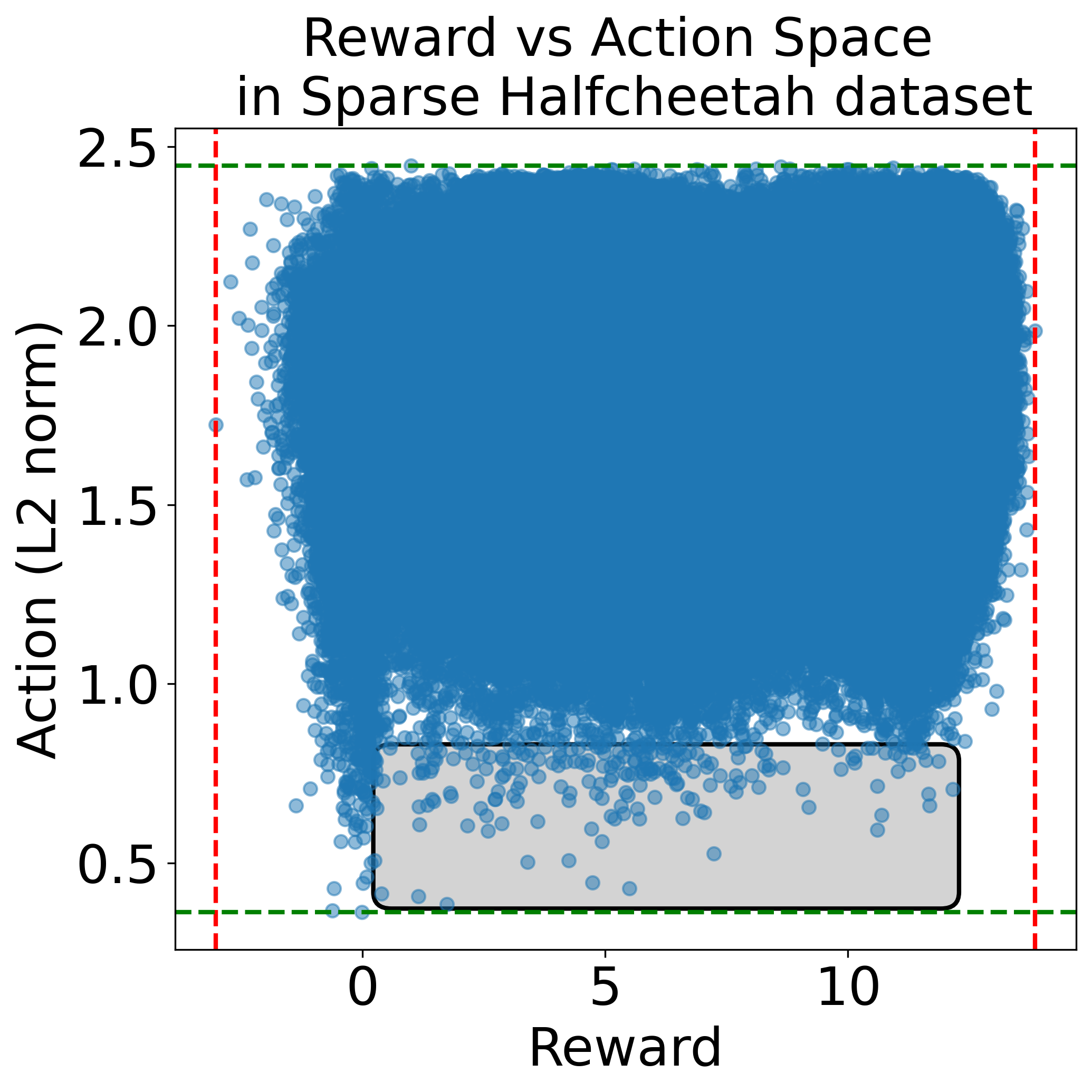}
        \caption{Action-reward space of the sparse halfcheetah dataset.}
        \label{fig:halfcheetah_sparse_a_w}
    \end{subfigure}
    \hspace{0.5cm}  %
    \centering
    \begin{subfigure}[b]{0.25\linewidth}
        \centering
        \includegraphics[width=\linewidth]{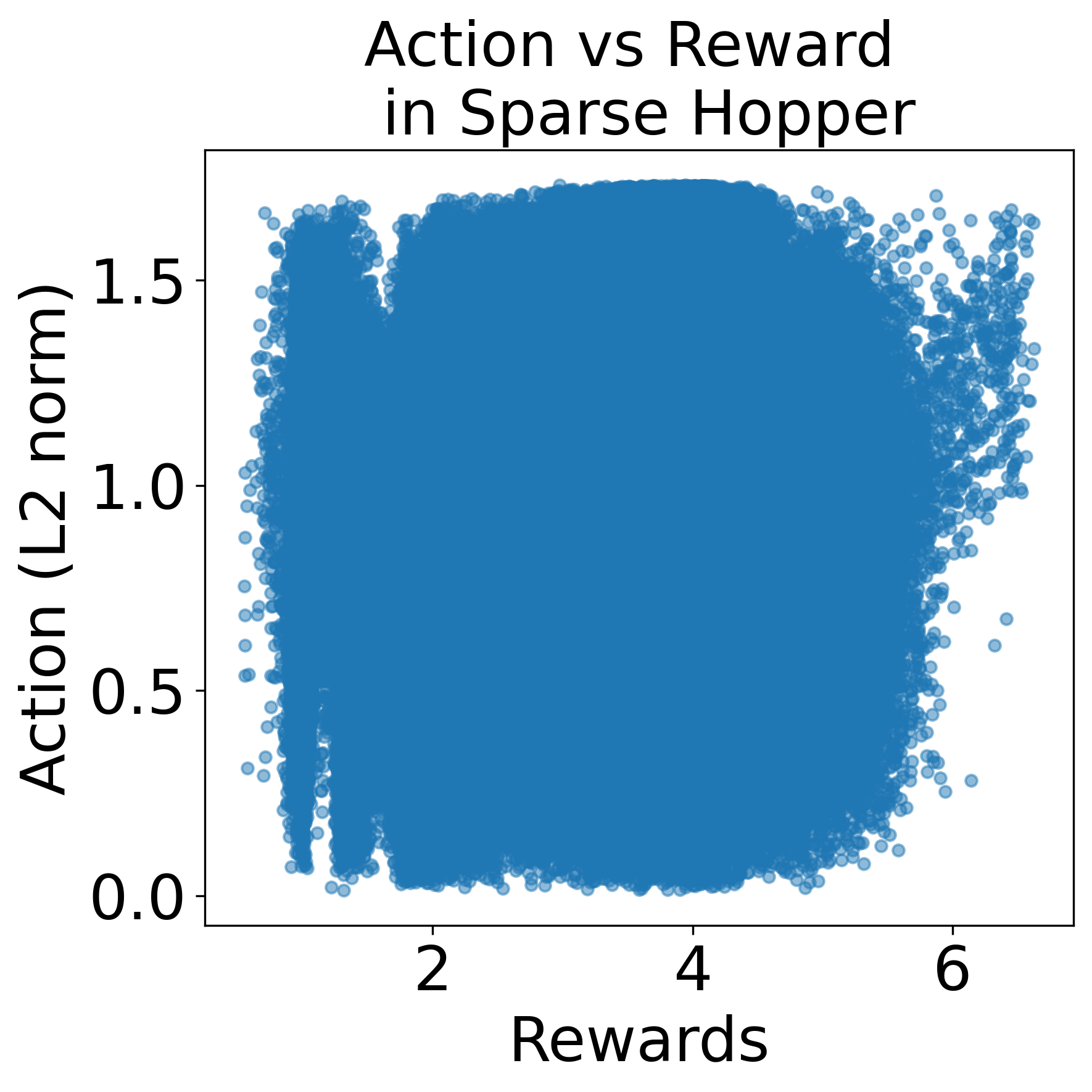}
        \caption{Action-reward space of sparse hopper dataset.}
        \label{fig:hopper_sparse_a_w}
    \end{subfigure}
    \hspace{0.5cm}  %
    \begin{subfigure}[b]{0.25\linewidth}
        \centering
        \includegraphics[width=\linewidth]{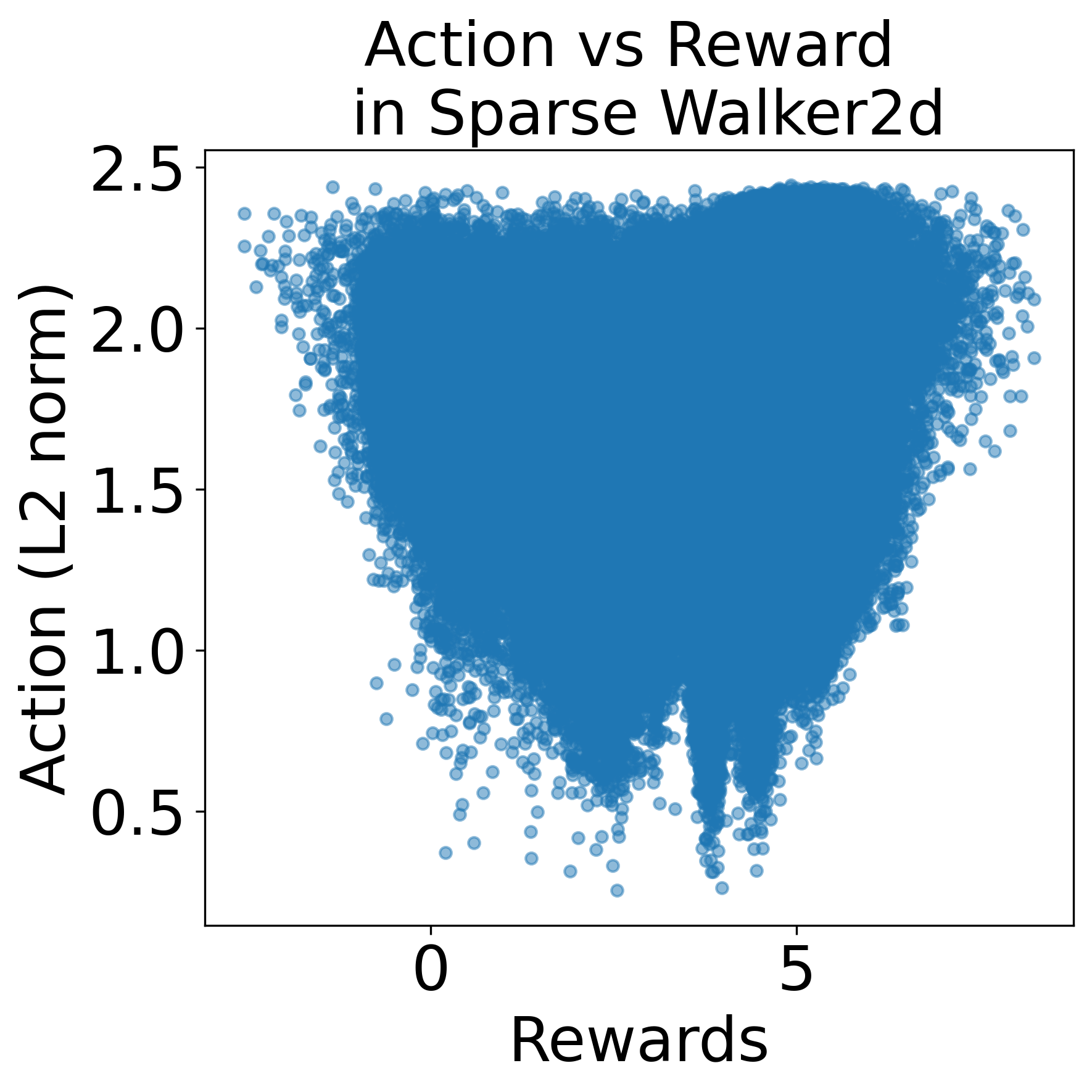}
        \caption{Action-reward space of sparse walker2d dataset.}
        \label{fig:walker2d_sparse_a_w}
    \end{subfigure}
    \caption{Action-reward space of sparse halfcheetah, walker2d, and hopper datasets. We compute the L2 norm of the actions to project it to 1D space for visualization purposes.}
    \label{fig:sparse_aw}
\end{figure*}

\begin{figure*}[!h]
    \centering
    \includegraphics[width=1\linewidth]{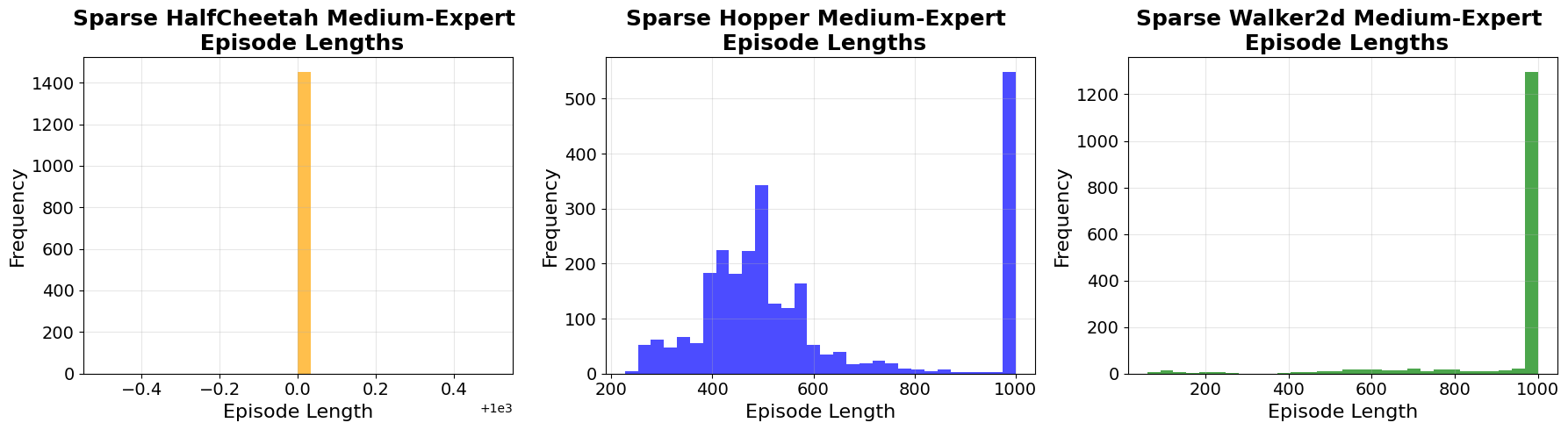}
    \caption{The distribution of episode lengths of the sparse D4RL datasets. halfcheetah's terminal signal is always 1000, while others vary a lot. This shows a deterministic termination signal in halfcheetah, unlike hopper and walker2d datasets.}
    \label{fig:terminal}
\end{figure*}

\clearpage
\section{t-SNE Plots for OOD Avoidance Behavior}
\label{sec:tsne}
\begin{figure}[htb]
    \centering
        \includegraphics[width=\linewidth]{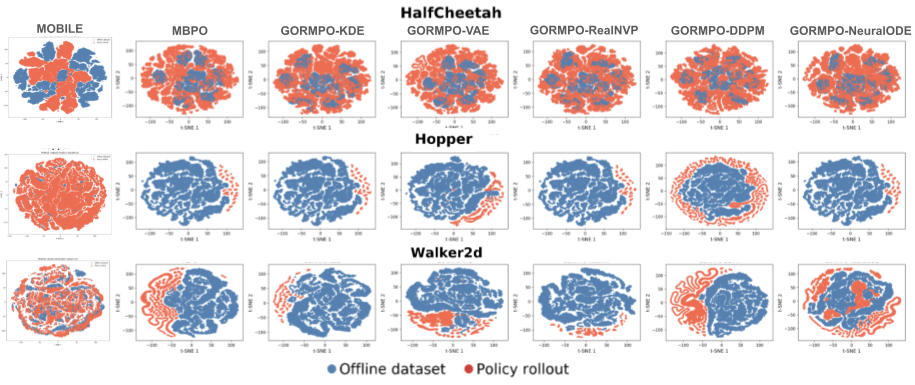}


     \caption{\textbf{t-SNE projections of equal-sized offline dataset and policy rollout $(s',a)$ pairs.} MBPO-based GORMPO rollouts expand beyond MOBILE rollouts in hopper and walker2d, while the MOBILE has more overlap and coverage. The best GORMPO models either symmetrically surround data support (GORMPO-DDPM in hopper) or stay near support without overlapping (GORMPO-RealNVP in walker2d). In halfcheetah, GORMPO-RealNVP depicts minimal deviation from data support, while MOBILE complements the data support.}
  
    \label{fig:tsne_combined}
\end{figure}

\section{Empirical Analysis of Diffusion Noise Predictions}\label{appendix:ddpm}

To rationalize the low OOD detection accuracy observed in our experiments, we analyze the statistical properties of the noise vectors predicted by the diffusion model. Our hypothesis is that the model's strong inductive bias towards the standard normal distribution masks the presence of Gaussian-perturbed anomalies.

As illustrated in Figure \ref{fig:noise_analysis}, the diffusion model consistently predicts noise $\epsilon$ that conforms tightly to the standard normal distribution $\mathcal{N}(0, 1)$, regardless of whether the input is ID or OOD.
\begin{itemize}
    \item Gaussian Adherence (Panel A): The predicted noise exhibits negligible deviation from normality, with near-zero skewness ($0.000$) and excess kurtosis ($-0.009$).
    \item Distributional Collapse (Panel B): Consequently, the noise distributions for ID and Gaussian-perturbed OOD inputs become statistically indistinguishable, yielding a Kullback-Leibler (KL) divergence of only $0.0246$.
\end{itemize}
This analysis suggests that the diffusion model projects both ID and Gaussian-perturbed OOD samples toward a similar Gaussian latent manifold, reducing separability at the aggregate distribution level. However, diffusion-based OOD detection often relies on sample-wise denoising or reconstruction discrepancies rather than solely on marginal statistics of the predicted noise. Therefore, overlapping aggregate noise distributions do not necessarily imply the complete absence of OOD signal. In our setting, the weak empirical separation may additionally stem from the strided ELBO approximation, or saturation effects in the density penalty, which can diminish sample-level discriminative information during policy optimization.

\begin{figure*}[!htb]
    \centering
    \includegraphics[width=0.8\linewidth]{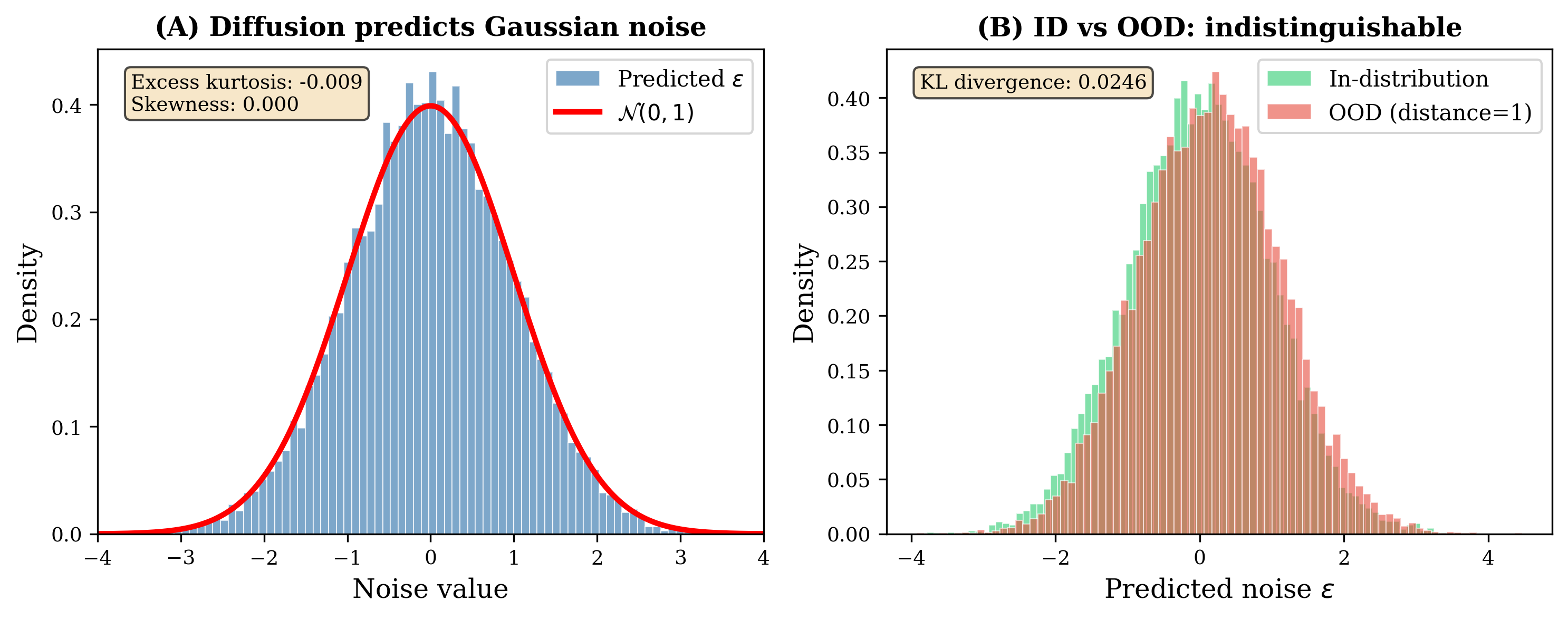}
    \caption{\textbf{Distributional overlap of predicted noise.} 
    \textbf{(A)} The predicted noise $\epsilon$ tightly conforms to a standard normal distribution $\mathcal{N}(0, 1)$, evidenced by negligible skewness ($0.000$) and excess kurtosis ($-0.009$). 
    \textbf{(B)} This Gaussian constraint causes substantial overlap between ID and OOD predictions. The low KL divergence ($0.0246$) confirms that the diffusion model projects distinct input distributions onto an indistinguishable Gaussian manifold.}
    \label{fig:noise_analysis}
\end{figure*}

\FloatBarrier
\section{Penalty Progression Plots for Sparse D4RL Datasets}
\label{sec:penalty}

\begin{figure*}[!h]
    \centering
    \includegraphics[width=1\linewidth]{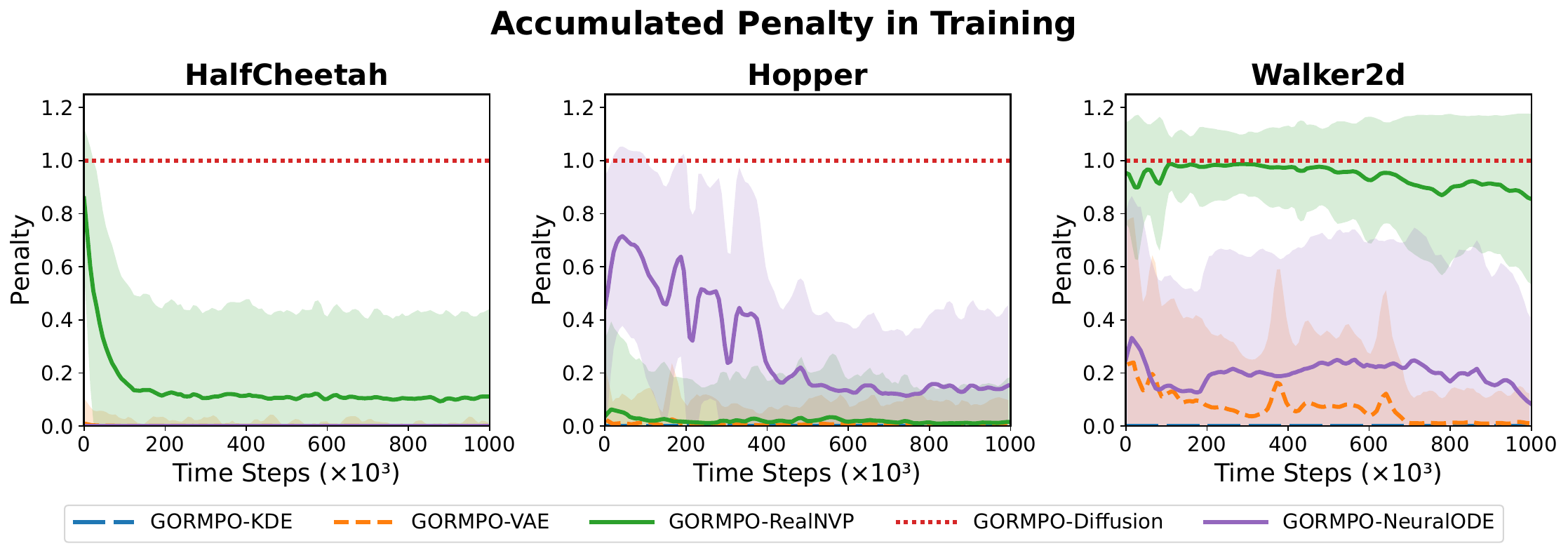}
    \caption{Mean and standard deviation of density-based penalty during training of MBPO-based GORMPO for 1000000 steps in the sparse D4RL datasets. }
    \label{fig:all_penalty}
\end{figure*}
\clearpage

%% file: alg_sample.tex



\begin{algorithm}[!htp]
\caption{Generative OOD-regularized Model-based Policy Optimization (GORMPO)}
\label{alg:cormpo}
\begin{algorithmic}[1]
\State \textbf{input} Jointly pre-trained dynamics model $\hat T_\phi(s'|s,a)$ and reward model $\hat r_\psi(s,a)$, pretrained density estimator $p_\theta(s',a)$, penalty weight $\lambda$, density threshold $\tau$.
\State \textbf{output} Learned policy $\hat{\pi}$
\State \textbf{Define penalized MDP}
\State \quad $\hat s' \sim \hat T_\phi(\cdot|s,a)$
\State \quad $\hat{u}(\hat s',a) = \tanh(\max(\tau - \log p_{\theta}(\hat s',a), 0))$
\State \quad $\tilde{r}(s,a) = \hat{r}(s,a) - \lambda\, \hat{u}(\hat s',a)$
\State \quad $\tilde{\mathcal{M}} \leftarrow (S,A,\hat{T}_\phi,\tilde{r}, \mu_0,\gamma)$
\State Initialize policy $\pi$
\For{$i=1, \dots, N_{\text{epoch}}$}
\State \textbf{Update $\pi$ with any MBRL model on $\tilde{\mathcal{M}}$}
\EndFor
\State \textbf{return} $\hat{\pi} \leftarrow \pi$
\end{algorithmic}
\end{algorithm}





